\documentclass{article}
 \usepackage[nonatbib,preprint]{neurips_2020}

\usepackage[numbers]{natbib}

\usepackage{hyperref}       % hyperlinks

\usepackage{math}
\usepackage{amsmath,amssymb,amsthm,mathrsfs,amsfonts,dsfont}
\usepackage{times}
\usepackage{graphicx} % more modern
\usepackage{subcaption}
\usepackage{algorithm}
\usepackage{algorithmic}
\usepackage{tikz}
\usetikzlibrary{arrows,calc,shapes,intersections}
\usepackage{xcolor}
\usepackage{mathtools}
\usepackage{paralist}
\usepackage{multirow}
\usepackage{thmtools,thm-restate}
\usepackage{array}
\usepackage{caption}
\usepackage{enumitem}
\usepackage[page,header]{appendix}
\newcommand{\A}{\mathcal{A}}
\newtheorem{theorem}{Theorem}
\newtheorem{lemma}[theorem]{Lemma}

\renewcommand{\S}{\mathcal{S}}

\newcommand{\interior}[1]{%
  {\kern0pt#1}^{\mathrm{o}}%
}
\definecolor{C0}{HTML}{1F77B4}
\definecolor{C1}{HTML}{FF7F0E}
\definecolor{C2}{HTML}{2ca02c}
\definecolor{C3}{HTML}{d62728}
\definecolor{C4}{HTML}{9467bd}
\definecolor{C5}{HTML}{8c564b}
\definecolor{C6}{HTML}{e377c2}
\definecolor{C7}{HTML}{7F7F7F}
\definecolor{C8}{HTML}{bcbd22}
\definecolor{C9}{HTML}{17BECF}

\author{%
  Yun-Shiuan Chuang \footnote[1]{Equal Contribution} \\
  Department of Computer Science and Psychology\\
  University of Wisconsin - Madison\\
  Madison, WI 53706 \\
\texttt{yunshiuan.chuang@wisc.edu} \\
    \And
    Xuezhou Zhang \footnote[1]{Equal Contribution} \\
  Department of Computer Science\\
  University of Wisconsin - Madison\\
  Madison, WI 53706 \\
\texttt{xzhang784@wisc.edu} \\
\And
    Yuzhe Ma  \\
  Department of Computer Science\\
  University of Wisconsin - Madison\\
  Madison, WI 53706 \\
\texttt{ma234@wisc.edu} \\
\AND
    Mark K. Ho \\
    Department of Computer Science and Psychology \\
    Princeton University\\
    Princeton, NJ 08540 \\
    \texttt{mho@princeton.edu} \\
\And
    Joseph L. Austerweil \footnote[2]{Co-Senior Author} \\
  Department of Psychology and Computer Science\\
  University of Wisconsin - Madison\\
  Madison, WI 53706 \\
\texttt{austerweil@wisc.edu} \\
\And
    Xiaojin Zhu \footnote[2]{Co-Senior Author} \\
  Department of Computer Science\\
  University of Wisconsin - Madison\\
  Madison, WI 53706 \\
\texttt{jerryzhu@cs.wisc.edu} \\
}
    \title{Using Machine Teaching to Investigate Human Assumptions when Teaching Reinforcement Learners}

\begin{document}
    \maketitle
    \begin{abstract} 
Successful teaching requires an assumption of how the learner learns - how the learner uses experiences from the world to update their internal states. We investigate what expectations people have about a learner when they teach them in an online manner using rewards and punishment. We focus on a common reinforcement learning method, Q-learning, and examine what assumptions people have using a behavioral experiment. To do so, we first establish a normative standard, by formulating the problem as a machine teaching optimization problem. To solve the machine teaching optimization problem, we use a deep learning approximation method which simulates learners in the environment and learns to predict how feedback affects the learner's internal states. What do people assume about a learner's learning and discount rates when they teach them an idealized exploration-exploitation task? In a behavioral experiment, we find that people can teach the task to Q-learners in a relatively efficient and effective manner when the learner uses a small value for its discounting rate and a large value for its learning rate. However, they still are suboptimal. We also find that providing people with real-time updates of how possible feedback would affect the Q-learner's internal states weakly helps them teach. Our results reveal how people teach using evaluative feedback and provide guidance for how engineers should design machine agents in a manner that is intuitive for people.
\end{abstract}

    \section{Introduction}
    People regularly teach other agents (e.g., children, pets, machines) in their environment using evaluative feedback (rewards and punishment). For example, Andrew is teaching his six-year-old daughter Jane to forage for wild berries. To do so, he first brings her to a bush with edible berries. What does he do next? If his goal is purely for her to eat some wild berries, then he has achieved his goal. However, to teach her to forage in a robust manner, he must provide her rewards to incentivize her to leave that bush and seek new ones. How do people teach agents using rewards and punishments, and how does their teaching depend on their knowledge of the internal dynamics of how the learner updates their beliefs?

        Although not always presented from this perspective, foraging is an example of the exploration-exploitation problem within reinforcement learning: How should an agent balance exploiting rewards based on their current knowledge while still exploring for berries is an example of an exploration-exploitation problem, which has been studied extensively for humans, animals, and idealized agents. Many natural agents (adults, children, animals, and other living creatures; \cite{cohen2007should,gopnik2020childhood,reid2016,stephens2007}) stop exploiting rewards at a state to explore, providing an implicit or explicit punishment to the agent, in a manner that is optimal within their ecological niche. However, often the mechanisms used to learn and change their behavior are tuned to assumptions of their environment, and do not adjust well to those in other environments. Although foraging itself (and learning while foraging) has been extensively studied within reinforcement learning and other mathematical frameworks, to the best of our knowledge, teaching others to forage is an open question. In this paper, we explore this question for teaching Q-learning agents a full policy in a task that requires teaching sub-optimal actions so that the learner learns what they should do in states they otherwise would not encounter (because they start close to their ultimate goal).

   Teaching the correct actions to take in a domain while exploring the domain is a social task involving the interaction of the teacher, a learner, and the environment. Researchers have examined how people teach others, formalized this process, and created automated methods to teach. One unifying computational framework across these areas is the Bayesian pedagogy framework, where the learner and teacher are Bayesian agents that assume both know the teacher is providing information to help the learner \citep{shafto2014}. However, this framework is mechanism-agnostic and assumes humans are doing ideal Bayesian updates, which can be a problematic assumption. Instead, we take the perspective of machine teaching, where we examine human teaching from the perspective of optimized teaching for a {\em particular learning mechanism} -- Q-learning algorithms.

    Q-learning is a family of model-free reinforcement learning algorithms that is known to learn the optimal policy as the agent interacts with the MDP 
\cite{sutton2018reinforcement}
-- undisturbed by a teacher.
Instead, we allow a human teacher to intervene in the agent-MDP loop, specifically by changing the reward signal, with the hope to teach the optimal policy faster. Concretely, we design multiple Q-learning agents (students) that differ in a number of critical parameters. If human teachers can teach some agents better than others, it means that those agents' parameters are closer to human teachers' assumption about how students learn.

    \section{Prior Work}
\label{sec:prior-work}

    From children to adults, when asked to teach, people provide different information than if they are simply asked to convey some information to another learner. For social situations that involve goals and rewards, assuming people are optimal agents within a Markov Decision Process captures human behavior well. For example, when asked to show how to do a task, people will take actions that are strictly unnecessary for completing the task, but convey information to a learner. However, if they are asked to do a task, they only do the necessary actions for completing the task ~\citep{ho2016}.

    Recent work in cognitive science and human-machine interaction has explored human teaching strategies and to what extent they are optimal. For example, work in Bayesian pedagogy~\citep{shafto2014} has shown that when teaching a range of simple concepts by example, people are capable of teaching other people near optimally. Similar research on linguistic pragmatics~\citep{grice1975,goodman2016} demonstrates that people's use of language reflects an intention to be optimally informative. Moreover, these findings on human teaching have been shown to generalize to more complex settings, such as sequential decision making~\citep{ho2016}. However, more complex settings also make it more likely that human teaching and learning processes are misaligned~\citep{ho2015,ho2019}, which motivates research into learning algorithms tailored to human teaching strategies~\citep{knox2009,macglashan2017}.

        Although the Bayesian pedagogy and social cognition literature provide useful perspectives on how people ought to teach others who know they are being taught, it does so for an idealized Bayesian agent assumed to be maximizing environmentally provided rewards. Albeit a useful assumption for examining human teaching, people do not teach or learn as an idealized Bayesian agent. In the 1990s, researchers learned how difficult it can be to train a reinforcement learner to complete simple tasks that contained necessary conditions which need to be completed by the learner completes the last steps of their task which are closer in state space to their current location \citep[e.g.,][]{ng1999}. Based on this intuition, recent work demonstrated that people fail to teach simple model-free and model-based reinforcement learners how to get from a start state to an end state while staying on a trail in a $3 \times 3$ Grid World \citep{ho2019}. The learners often pick up on ``positive net reward cycles" which enable them to get arbitrarily large reward while not completing the task.

\textbf{Computational Teaching:} Since computational teaching was first proposed in~\cite{shinohara1991teachability}, the teaching has been studied in various learning settings, see \cite{DBLP:journals/corr/ZhuSingla18} for a recent survey. 
Of particular interest to us are works on teaching online learners such as Online Gradient Descent (OGD)~\cite{liu2017iterative,lessard2018optimal}, active learners~\cite{hanneke2007teaching,peltola2019machine}, and other sequential learners ~\cite{hunziker2019teaching,mansouri2019preference,zhang2019online, jun2018adversarial, ma2018data,wang2018data}. The optimal control formulation is required when the learner becomes sequential.
Several recent works studied teaching on Inverse Reinforcement Learning (IRL)~\cite{DBLP:conf/nips/TschiatschekGHD19,DBLP:conf/ijcai/KamalarubanDCS19,brown2019machine,haug2018teaching,cakmak2012algorithmic}, where learner learns from teacher demonstration.
Finally, computational teaching in reinforcement learning have been studied recently \cite{zhang2020adaptive, rakhsha2020policy,ma2019policy}, where teacher teach via rewards and/or state transitions. Our work instead focuses on using computational teaching theory to understand how humans teach.

\section{Preparing a Family of Reinforcement Learners for Teaching}\label{sec:mach-teach}

In the teaching problem of reinforcement learning, we study the interaction between three entities: the RL agent (student), the teacher, and the underlying environment. In this work, we assume that the environment is a \textbf{rewardless} Markov Decision Process (MDP) parametrized by $M = (\mathcal S, \mathcal A, P, \mu_0)$ where $S$ is the state space, $\mathcal A$ is the action space, $P: \mathcal S\times \mathcal A \times \mathcal S \rightarrow \R$ is the transition probability, and $\mu_0: \mathcal S\rightarrow \R$ is the initial state distribution.
\subsection{A Family of Q-Learners}
In our human experiments, we focus on a family of Q learners, denoted by $\mathcal L$, that is widely studied as potential theory of mind for humans, among which we describe two variants: (1) standard Q-learning and (2) Action Signaling (AS) \cite{ho2019}. The two learners mainly differ in their internal knowledge representation and how they perform learning updates.
 
\textbf{A Q-learning agent} stores an estimate of the \textit{Q table}, $Q: \S\times \A\rightarrow \R$, which approximates the future cumulative rewards that the agent can receive after perform an action $a\in\A$ in a state $s\in\S$. The learning update rule of Q-learning is defined by 2 parameters, the learning rate $\alpha$, and the discounting factor $\gamma$. $\alpha$ determines how aggressive the learner updates the current belief given the new experience, and $\gamma$ indicates how much the learner value future rewards compared to immediate rewards. Specifically, given a new piece of experience $e_t = (s_t, a_t, s_{t+1}, r_t)$, Q-learning only updates the $(s_t,a_t)$ entry of its Q table as
\begin{equation}
Q_{t+1}(s_t,a_t) = (1-\alpha)Q_t(s_t,a_t) + \alpha (r_t + \gamma \max_{a'} Q_t(s_{t+1},a'))
\end{equation}

\textbf{An Action Signaling (AS1) agent} stores a multinomial distribution over actions for each state, representing its current belief distribution of the optimal action \cite{ho2019}. In this paper, we represent the multinomial distribution also with a table, $Q: \S\times \A\rightarrow \R$, where $\sum_{a\in\A} Q(s,a)=1$. The learning update of an AS agent is defined by a single parameter, the learning rate $\kappa$, which serves the same role as $\alpha$ in Q-learning. Given a new piece of experience $e_t = (s_t, a_t, s_{t+1}, r_t)$, an AS agent will update the belief distribution over the current state $s_t$. If $r_t>0$, the probability w.r.t. $a_t$ will increase, whereas if  $r_t<0$, the probability w.r.t. $a_t$ will decrease. In contrast to Q-learning, an AS agent does not make use of the next state $s_{t+1}$, and does not aim at optimizing long term reward. Specifically,
\begin{equation}\label{eq:as1}
Q_{t+1}(s_t,a) =
{ Q_{t}(s_t,a)  e^{1_{[a=a_t]} \kappa r_t}
	\over
	\sum_{b \in A} Q_{t}(s_t,b) e^{1_{[b=a_t]} \kappa r_t} }
, \; \forall a \in A.
\end{equation}
We further observe that even without the normalization step, $\argmax_a Q_t(s,a)$ will not change. Therefore, the learning update rule \ref{eq:as1} can be equivalently changed to $Q_{t+1}(s_t,a) = Q_{t}(s_t,a)  e^{1_{[a=a_t]}} \kappa r_t$. We again observe that if instead of $Q_t(s,a)$, we store $Q'_t(s,a)=\log Q_t(s,a)$, then the update rule becomes $Q'_{t+1}(s,a) = Q'_t(s,a) +\kappa r_t$. Finally, we can drop the $\kappa$ factor, since it's a constant scaling on all $(s,a)$-pairs, and we get $Q'_{t+1}(s,a) = Q'_t(s,a) + r_t$. Therefore, we have established that \textbf{AS is equivalent to a variant of Q-learning}, which simply stores the sum of rewards received on each $(s,a)$-pair. This motivates us to define another variant of the AS agent that instead of taking the sum of rewards, takes the average of rewards on each $(s,a)$-pair.

\textbf{An Action Signaling (AS2) agent} stores the average of rewards on each $(s,a)$-pair, whose update rule can be equivalently written as
\begin{equation}\label{eq:as2}
Q_{t+1}(s,a) = (1-\frac{1}{n_t(s,a)}) Q_t(s,a) +\frac{1}{n_t(s,a)} r_t
\end{equation}
where $n_t(s,a)$ denotes the number of times the current $(s,a)$-pair has been visited. Similar to AS1, AS2 do not take into account of future rewards, and is in fact equivalent to a Q-learning agent with $\gamma = 0$, and time-varying learning rate $\alpha_t = 1/n_t(s,a)$. For those who are familiar with the classic RL/multi-armed bandit literature, the AS2 agent is equivalent to a multi-armed bandit algorithm that treats the MDP as $S$ parallel $A$-arm bandits.

We further assume that all learners will behave according to the \textbf{$\mathbf{\epsilon}$-greedy policy} w.r.t. the current Q table, i.e.
\begin{equation}
a_t = \pi_t^\epsilon(s_t) = \left\{
\begin{array}{ll}
\argmax_a Q_t(s_t, a), & \mbox{ w.p. } 1-\epsilon \mbox{, break ties uniformly} \\
\mbox{uniform from } A, & \mbox{ w.p. } \epsilon.
\end{array}
\right.
\label{eq:explorationpolicy}
\end{equation}
We distinguish two teaching settings:
In a \textit{white-box} teaching setting, we assume an \textit{omniscient teacher}, who has knowledge of the underlying MDP $M$ and the learning agent's parameters. It also observes the current interaction  $s_t,a_t, s_{t+1}$ as well as the internal state of the agent $Q_t$, before providing the reward signal $r_t$; In a \textit{black-box} teaching setting, we assume that the teacher can still observe $(s_t,a_t, s_{t+1},Q_t)$ but does not know the precise learner update rule, and therefore cannot accurately predict the consequence of the current reward signal.

The teacher's goal is to drive the learner to learn a target policy $\pi^\dagger$.
For example, the teacher may want the learner to always perform a specific action $a^\dagger \in A$ at state $s^\dagger \in S$.
This example goal can be expressed as a target set of Q tables that satisfy
$\mathcal Q^\dagger := \{Q: Q(s^\dagger, a^\dagger) > Q(s^\dagger, a),\forall a\neq a^\dagger\}$.
The teaching succeeds if the learner's $Q_t$ falls into $\mathcal Q$ at some time step $t$, in which case the teaching process terminates. We are interested in calculating a teaching strategy that achieves the teaching goal with the fewest time steps.

\subsection{An Optimal Teacher can Teach the Family of Q-Learners Equally Fast}
One of our main observations is that the optimal teaching problem in the white-box teaching setting forms a higher level teaching MDP $\mathcal N=(\Xi, \Delta, \rho, \tau)$:
    \begin{itemize}[leftmargin=*, nolistsep]
    	\item The teacher observes the \textbf{teacher state} $\xi_t \in \Xi$, which jointly characterizes the environment and the learner at time $t$:
    	$
    	\xi_t := (s_t, a_t, s_t', Q_t).
    	$
    	\item The teacher's action space consists of all possible rewards $r_t \in \R = \Delta$.
    	\item The teacher receives a constant cost of $\rho_t = 1$ for every time step before the teaching goal is accomplished.
    	\item The teaching state transition probability is specified by $\tau(\xi_{t+1} \mid \xi_t, r_t)$.
    	The new attack state $\xi_{t+1}=(s_{t+1}, a_{t+1}, s'_{t+1}, Q_{t+1})$ is generated as follows: $s_{t+1}$ is copied from $s_t'$ in $\xi_t$; $a_{t+1} \sim \pi^\epsilon_{t+1}(s_{t+1})$; $s'_{t+1}\sim P(\cdot | s_{t+1}, a_{t+1})$; 
    	$Q_{t+1}$ is the learner's updated Q table.
    \end{itemize}
   	The optimal teaching strategy is one where the teacher minimize its cumulative attack cost,
    $\sum_{t=0}^{T} \rho_t$, s.t. $Q_T\in \mathcal{Q}$.
    Due to the randomness of the MDP as well as the learner's behavior policy, this quantity is a random variable, and thus we instead minimize its expected value.
    Formally, the teacher seeks a time-invariant \textbf{teaching policy} $\phi^*$: 
    $
    \phi^* = \argmin_{\phi: \Xi \mapsto \Delta} \mathbb E_{\mathcal N} \left[ \sum_{t=0}^{T} \rho_t \mbox{ , .s.t }Q_T\in \mathcal Q^\dagger\right].
    $
    We will also denote the shortest expected time step to achieve the teaching goal as the \textbf{teaching dimension} for the corresponding teaching problem instance, i.e. 
    \begin{equation}
    TD(M,L,Q_0, \pi^\dagger) = \min_{\phi: \Xi \mapsto \Delta} \mathbb E_{\mathcal N} \left[ \sum_{t=0}^{T} \rho_t \mbox{ , s.t. }Q_T\in \mathcal Q^\dagger\right]
    \end{equation}
    The optimal control formulation allows us to computationally solve for the optimal teaching strategy on any specific teaching instance, defined by the MDP, learner type and target policy to be taught, using any optimal control solver. In this paper, we use Twin Delayed DDPG (TD3)~\cite{fujimoto2018addressing}, a state-of-the-art Deep Reinforcement Learning (DRL) algorithm that solves continuous control problems.
    
    Our second theoretical insight is that the any learner in the Q-learning family  $\mathcal L$ has the same teaching dimension. Specifically,
    \begin{theorem}\label{thm: equivalence}
    	For any two teaching instances defined by $(M, L, Q_0, \pi^\dagger)$ and $(M, L', Q'_0, \pi^\dagger)$, where $L$ and $L'$ are two Q-learning agents, if $Q_0$ and $Q'_0$ satisfies that $Q_0(s,a)\geq Q_0(s,a')$ if and only if $Q'_0(s,a)\geq Q'_0(s,a')$ for all $s\in\S$, $a,a'\in \A$, then $TD(M, L, Q_0, \pi^\dagger) = TD(M, L, Q_0, \pi^\dagger)$.
    \end{theorem}
In other words, under the white-box setting, an optimal teacher can teach a learner in the family $\mathcal L$ equally fast.

\section{Human Teaching on the Family of Q-Learners}
But our human teachers do not really know what algorithm or parameters the students use.
This allows us to probe the human assumption on the student: we expect that when an student algorithm is closer to human preconception of a student, the human teacher will be better at teaching that student.

    How does human teaching compare to machine teaching in a task where a person needs to teach a reinforcement learner a full policy when the learner starts one action away from its environmental goal? Figure ~\ref{fig:task-UI} presents an idealized exploration-exploitation scenario. In the scenario, the dog's initial state is one tile left of its goal (the door which symbolizes its home). The teaching goal is to teach the dog to go home from every tile. Whenever the dog reaches its home, the next step it starts on the initial state. To teach the dog to get home from every square, the teacher must incentivize the dog to explore, which results in the dog moving away from its ultimate goal: It must punish moving to the goal and/or reward going away from the goal. Once the dog has gotten all the way to the leftmost tile, the teacher can begin to ``undo" their prior teaching and teach the dog to go right. Depending on how much reward/punishment the dog is given to explore, it may take many steps (or even be impossible given feedback with finite limits) to teach the dog to move right at every tile.

    To analyze the difficulty of this task, we used the machine teaching method described in Section ~\ref{sec:mach-teach} to provide the optimal teaching policy. The learner is a Q-learner that uses an $\epsilon$-greedy policy to select its action and it starts indifferent $Q_0(s,a) = 0$ for every feasible state-action pair $(s,a)$. For concreteness, we assume the learner is parametrized as follows: $\alpha=0.1$, $\gamma=0.9$, and $\epsilon=0.1$. Given this configuration, the optimal teaching policy is precisely the one discussed above (get the dog to move all the way left with minimal reward/punishment and then teach it to move right at every state). The number of steps to teach is a random variable whose outcomes depend on what actions are selected when the dog is indifferent and when suboptimal actions are taken. Thus, we approximated the expected number of steps via Monte Carlo (simulating teaching dogs using the optimal teaching policy and recording the number of steps it took for the full target policy to be taught). On average, the optimal teaching policy takes 11 steps to teach the dog the full target policy. It is 11 regardless of the value of $\alpha$ and $\gamma$ (assuming they are within the values that allow $Q$-learning to converge). Changing $\epsilon$ affects the optimal teaching policy and expected number of steps quantitatively, but the optimal teaching policy follows the same qualitative procedure as before. So, we set $\epsilon=0.1$ for the remainder of the paper.

    Although our machine teaching results suggest all Q-learners should be trainable with the same expected number of steps, can people train all learners? Are there some learners that are easier for people to train? As discussed earlier, recent work has found that traditional parametrization for Q-learners are extremely challenging for people to train on even simple tasks and the action-signaling model (which we proved is a special-case of Q-learning) is much easier for people to train \citep{ho2019}. However, they did not test a large set of parametrization for the discount rate ($\gamma$) and some of their analyses suggest that a Q-learner might be trainable when the discount rate is small. This is consistent with other work on hyperdiscounting by Taber and colleagues \citep{knox2009} that reinforcement learners with small discount rates are easier for people to teach. In this study, we investigate the role of the discounting parameter for a simple teaching task and examine how giving people the internal dynamics of the learning update affects their ability to teach efficiently.

    \subsection{Experimental Design}
    {\em Participants.} We recruited 791 participants through Amazon MTurk. 
The number of participants was chosen {\em a priori} based on the authors' intuition from similar previously conducted studies.  
We excluded 42 participants (11 for failing to complete the task, 19 due to experiment error, and 12 who selected ``do nothing'' for at least 36 steps). 
In addition, the dogs that were trained successfully with steps less than the optimal length (the number of steps the optimal teaching policy needs) were excluded from the analysis. For example, if the subject gave punishment on every step and the dog luckily went to the left at every tile all the way from the rightmost tile to the leftmost tile, the dog could be trained in just 4 steps. These data were excluded because they did not really reflect a success in training, but more likely a misunderstanding of the goal and luck. This composes 9.03 \% of the total dogs (203 out of 2247 dogs).  
    
    {\em Interface/Stimuli.} The dog training task took place in a 4 $\times$ 1 MDP with an absorbing state on the right. Figure ~\ref{fig:task-UI} shows the visual interface that the participants interacted with. Four states were represented by four tiles that the dog could walk on, and the absorbing state was represented by a door. At each step, the dog could only go right or left from a tile to the nearby tile or the door.  If the dog went left at the leftmost tile, the dog would stay at the same tile. If the dog went right at the rightmost tile, the dog would go to the door and then placed back to  the rightmost tile. The training for a dog ended if the dog had successfully learned the target policy, or the dog had already taken 40 steps but still had not learned the target policy. Once the training for the dog ended, a new dog with a different color would be shown and the participant was asked to train the new dog. The internal states were displayed as two rows provided by a ``brain scanner", which corresponded to the Q table and current optimal policy. At the beginning of each dog, the dog was placed at the rightmost tile with an initial Q table where the Q value of each state-action pair is zero. Participants responded via a continuous slider (feedback of -1 to 1), or could select a button to ``do nothing'' (feedback of zero). 
    \begin{figure}
        \centering
        \includegraphics[width=0.7\textwidth]{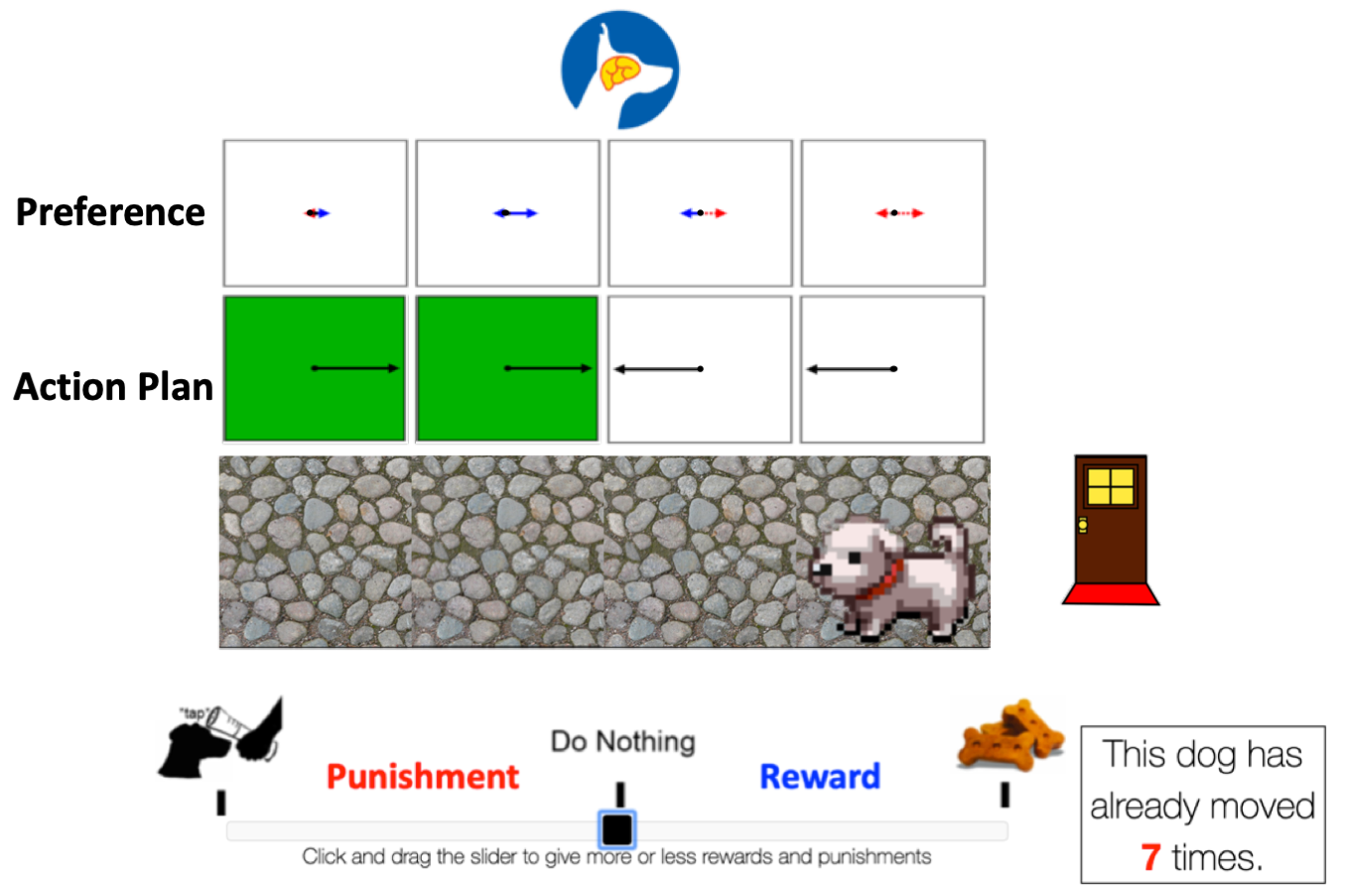} \\        
        \caption{The dog training task took place in a small garden that was composed of four tiles and a door on the right. The dog's current Q table (labeled ``Preference") and policy (labeled ``Action Plan") were shown above the garden. In the Q table, there were four cells corresponding to the Q values of each of the four tiles. In each cell, there were arrows pointing to the left/right, which encoded the Q value of moving to the left/right at that tile. A blue solid arrow encoded positive Q value, whereas a red dotted arrows encoded negative Q value. The policy was based on the Q table, where the arrow pointed to the direction that the dog would go in an exploitation step. When the dog's policy at a tile dictated going towards the door, the background of the cell turned green, indicating the policy at that state matched the target policy. After the dog took its action, the participant was asked to provide feedback to the dog. The participant could used the slider to select the feedback value or clicked the ``do nothing'' button to give a zero feedback. After the participant chose the feedback, the dog's Q table and target would get updated accordingly, and the dog would perform its next action. In the bottom-right corner, there as a step counter that recorded the steps that the dog has taken so far.}
        \label{fig:task-UI}
    \end{figure}
    
    {\em Procedure.} Participants were asked to teach three dogs to go home from any tile. They were given an extensive quiz about the instructions and could not continue until every question was answered correctly (see Supplementary Materials). There were two between-subjects conditions: {\em learner type} $\times$ {\em learning dynamics}. Each participant was assigned to either one of the four traditional Q-learners ($\gamma \in  \{0.0, 0.1, 0.45, 0.9 \}, \alpha= 0.9 $) and two action-signaling models. 
Learning dynamics were either shown (the whitebox setting) or not shown (the blackbox setting) to the participant. if they were shown, then the internal states displayed to the participant were synchronized to the current position of the slider. Otherwise the internal states only updated after the participant provided feedback. On each trial, the dog would move from one state to another. It followed an $\epsilon$-greedy policy ($\epsilon=0.1$), where the policy was the optimal one for the current Q-matrix. When a random action was selected for the dog, a squirrel would appear in the direction of that action (participants were told that when a squirrel appeared, the dog would move towards it no matter what its internal states were). After the dog moved, the participant would respond by dragging the feedback slider or hitting the ``do nothing'' button. Training concluded when the participant successfully taught the target policy or 40 steps had completed. Participants trained three dogs and then was given a short survey to ensure that they treated the slider symmetrically and gather standard demographic data.

%    Participants were randomly assigned to one of the six experimental conditions. The six experimental conditions were formed by three learner types (Q learner, AS1 learner, AS2 learner) crossed with two slider conditions (whether the participants can see the effect of the feedback on the Q table when they are sliding the slider to decide the feedback value). Participants all read the same instruction and took the same quiz regardless of which condition they were in. After the quiz, the participants began to train the dogs. Each participant were asked to train three dogs. The training of a dog was completed either (1) if the dog successfully learned the target policy (i.e., the dog's action plan pointed toward the house at every tile), or (2) if the dog failed to learned the target policy after it had taken 50 steps. After a participant finished training three dogs (regardless of success or failure), he/she was asked to complete a post-task survey that included the demographic questions (e.g., age, gender, educational level) and their experience with dog in the real life. TODO: I am not sure how detailed I should write about the post-task survey.

    \subsection{Results and Discussion}	
	\begin{table}
    \caption{The Descriptive Statistics of the Human Experiment}
    \label{table:exp1-desc-stats}
	 \begin{tabular}{m{6em} m{1.5em} m{4.5em}  m{10em}  m{11em}} \\[0.5ex]
	 \hline
	 Learner Type & Slider Sync & \# Subjects & Success Rate  (95\% CI) & Avg Steps when Successful (95\% CI)\\
	 \hline
	 AS1 & on & 66 & 30.1 ([23.4, 36.7]) & {\bf 20.6 ([18.8, 22.2])}\\
%	 \hline
	 AS1 & off & 71 & 33.3 ([26.7, 40.0]) & 24.3 ([22.6, 25.9])\\
%	 \hline
	 AS2 & on & 68 & 26.7 ([20.2, 33.1]) & 21.8 ([19.6, 23.9])\\
%	 \hline
	 AS2 & off & 67 & 26.3 ([20.1, 32.6]) & 21.2 ([19.2, 23.2])\\
%	 \hline
	 Q ($\gamma = 0.00$) & on & 57 & {\bf 51.3 ([43.4, 59.2])} &  20.9 ([19.5, 22.3]) \\
%	 \hline
	 Q ($\gamma = 0.00$) & off & 57 & 47.4 ([39.4, 55.3]) & 22.3 ([20.8, 23.8]) \\
%	 \hline
	 Q ($\gamma = 0.10$) & on & 57 & 47.1 ([39.3, 54.9]) & 21.3 ([19.7, 22.9])\\
%	 \hline
	 Q ($\gamma = 0.10$) & off & 55 & 46.5 ([38.4, 54.7]) & 22.5 ([20.6, 24.4]) \\
%	 \hline
%	 \hline
	 Q ($\gamma = 0.45$) & on & 57 & 36.9 ([29.4, 44.5])  & 21.5 ([19.5, 23.5])\\
%	 \hline
	 Q ($\gamma = 0.45$) & off & 61 & 35.2 ([27.8, 42.5]) & 23.8 ([21.8, 25.8]) \\
%	 \hline
	 Q ($\gamma = 0.90$) & on & 68 & 21.9 ([16.0, 27.9]) & 22.0 ([19.9, 24.2])\\
%	 \hline
	 Q ($\gamma = 0.90$) & off & 65 & 24.6 ([18.4, 30.8]) & 25.2 ([22.9, 27.4]) \\
	 \hline
	\end{tabular}
	\\[10pt]
	\caption*{Learner Type: the learner type of the dogs being trained. Slider Sync: whether the subjects can see the effect of the feedback on the Q table when they are sliding the slider to decide the feedback value. \# Subjects: The number of subjects that were included in the analysis. Success Rate (\%): The proportion of dogs that were successfully trained. Average Steps when Successful: The average steps the subjects used to train the dogs when the training is successful.  AS1: The AS1 agent. AS2: The AS2 agent. Q: The Q-learning agent.}
	\end{table}

    \begin{figure}
        \centering
        \includegraphics[width=.95\textwidth]{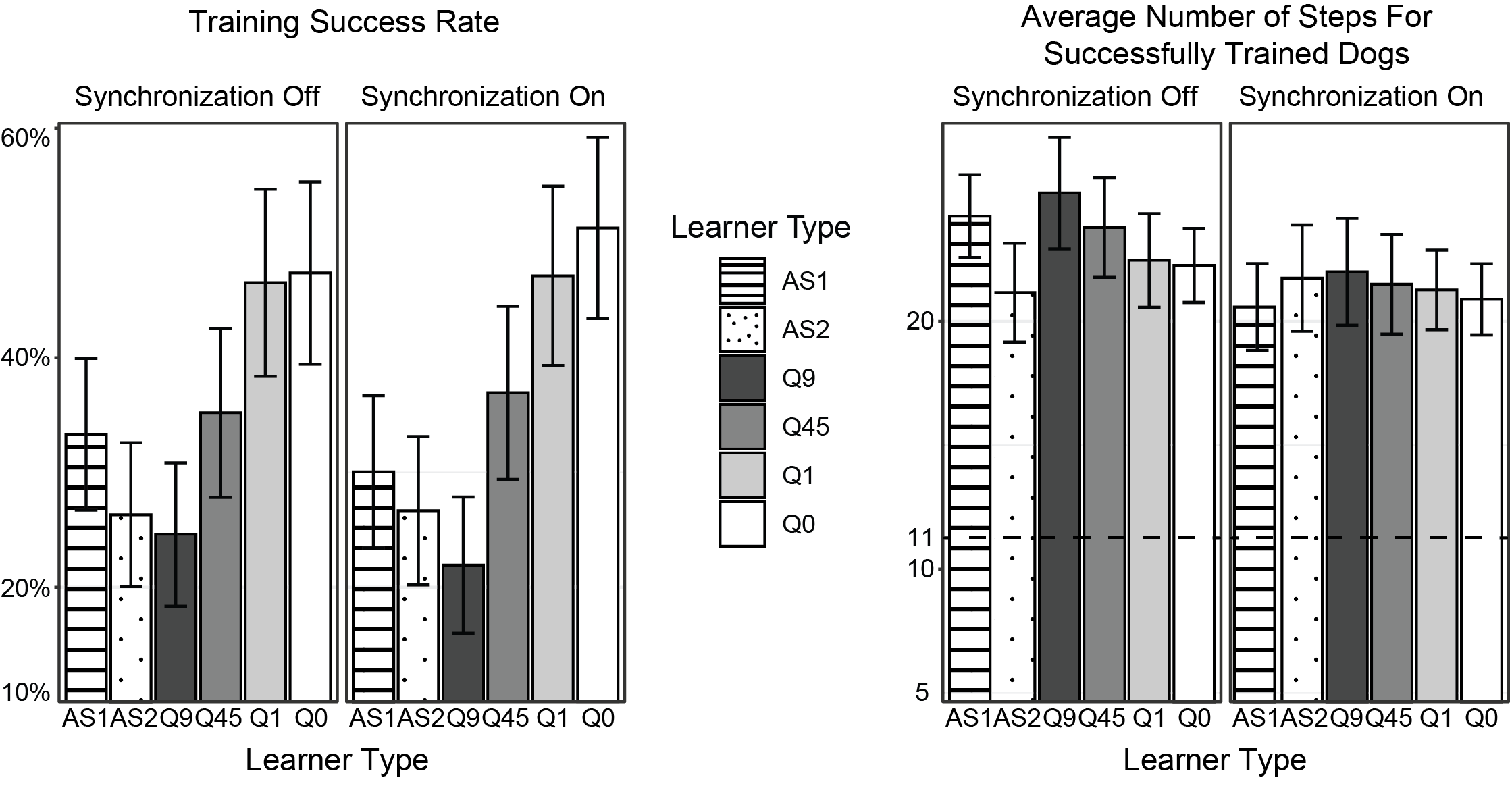} \\
        \caption{(Left) Participant success at training the full policy for different learner types and internal dynamics conditions (synchronization on vs. synchronization off). (Right) Average number of steps for successfully trained dogs for different learner types and internal dynamics conditions  (synchronization on vs. synchronization off). Note: error bars denote 95\% confidence intervals and the dashed line corresponds to the average number of steps for an optimal teacher. AS1: The AS1 agent. AS2: The AS2 agent. Q9: The Q learning agent with $\gamma$ = 0.9. Q45: The Q learning agent with $\gamma$ = 0.45. Q1: The Q learning agent with $\gamma$ = 0.1. Q0: The Q learning agent with $\gamma$ = 0. 95\% CI: the 95\% confidence interval.}
        \label{fig:expt-comb-res}
    \end{figure}

    Figure ~\ref{fig:expt-comb-res} shows the success rate of participants at training the full policy and the average number of steps taken when a dog was successfully trained. To analyze the success rate, we fit a mixed-effects logistic regression model where learner type, internal dynamics, and their interaction were fixed effects, and participant was a random effect. Learner type significantly influenced the success rate for a participant ($\chi^2(5) = 20.7, p < 0.001$), but the internal dynamics did not ($\chi^2(1)=0.66, p=0.42$). To analyze the average number of steps for successfully trained dogs, we fit a mixed-effects linear regression where learner type, internal dynamics, and their interaction were fixed effects, and participant was a random effect. There also was no interaction ($\chi^2(2) = 1.64, p = 0.90$). Both learner type and internal dynamics significantly affected the average number of steps a person took to train a dog successfully ($F(5,341.76)=2.21, p = 0.05$ and $F(1, 248.82)= 4.19, p < 0.05$, respectively).

    How well did people teach the dogs? Clearly they were suboptimal -- at best, the average success rate was 51\% and when successful took about 21 steps (optimal has a success rate of 100\% and takes an average of 11 steps). To analyze whether people adapted their teaching over time, we conducted the following permutation test. For each participant, we created 1000 simulated Q-learners. For each simulated Q-learner and state-action pair, we sampled a random permutation of the feedback provided by the participant for that state-action pair. This provides a Monte Carlo test of whether human teaching adapted with learning. Participants clearly adapted their behavior: none of the simulated Q-learners for any participant ever learned the target policy when the feedback for each action-state pair randomly permuted. Thus, human teaching is quite successful in this task (albeit still suboptimal).

In sum, we observed that human teachers are suboptimal: they did not always succeed in teaching the policy to an agent; and when they did, they required more steps compared to the 11 steps of the optimal teaching algorithm in Section 3.
This is not surprising: we never expected humans to be optimal teachers.
Our important discovery is that human teachers' ``favorite student'' is Q0, which differs from other Q-learning agents in two key parameters: the smallest possible discounting factor $\gamma=0$, and a large Q-update step size $\alpha=0.9$.
What does this mean?  The favorite student is one which is \emph{predictable by the teacher} (no complications from discounted future rewards because $\gamma=0$) and amenable to \emph{myopic control from the teacher} (large $\alpha$ emphasizes immediate teacher reward).
We thus suggest that human teachers assume such predictable, myopically controllable students when they teach.
On the other hand, whether the teacher can see the ``guts'' of the agent (whitebox / blackbox setting with effect of reward slider on / off, respectively) turns out to not be very important for the teacher.

    \section{Conclusions}
    In this paper, we conducted the first investigation of how people teach learners to solve the exploration-exploitation tradeoff, which is a common problem in everyday life. We did so by first formulating it as a machine teaching problem where the learner is assumed to be a Q-learner. We then ran a behavioral experiment of an idealized scenario to see how well people train Q-learners to solve the exploration-exploitation tradeoff when the teaching goal is an entire policy. We found that people are suboptimal, but quite successful, and that teaching was the easiest when the Q-learner had a small discount rate $\gamma$ and large update step size $\alpha$. We also showed that revealing how possible feedback would change the learner's beliefs weakly helped people teach in a more efficient manner.

    Given that our study is the first to explore how people teach learners to solve an exploration-exploitation tradeoff, there are potential concerns with the generalizability of the study. For example, it is unclear whether our results would generalize to other environments. It is also unclear how robust the results are to other types of reinforcement learners. For example, it seems plausible that a model-based reinforcement learning method would be easier to train on our idealized task. However, prior work suggests that model-based learners would struggle in more complex environments where positive net cycles are easily taught.

\newpage
\section{Broader Impacts}
Over the last few decades, human-machine interactions have become a regular part of everyday life. Although these interactions are often successful, they also often fail, usually due to the machine agent not properly understanding the person's feedback. When they do succeed, it is often because people have learned to interact with the machines, rather than machines learning or being created to interact with people in a manner that is natural to people. Not everyone can learn to interact with machines, especially those who have not kept up with current technology. This provides a large gap in the tools available to different people, with those who are disadvantaged often struggling to learn how to interact properly with the machines due to lack of experience.

One main method people interact with one another and machines is to teach by providing evaluative feedback. Despite much progress, scientists neither fully understand the mechanisms underlying this behavior nor have a formal description that can guide or be ``plugged into" machine agents. This paper provides the first mechanism-level analysis of how a learner should be taught, as well as an analysis of how people teach via evaluative feedback in an idealized exploration-exploitation task. It provides guidance for how engineers should parametrize machine agents who use Q-learning while interacting with people. This is a step towards the larger goal of natural human-machine interaction that is equally accessible to all.

\bibliography{ms}
\bibliographystyle{plain}

\newpage
\appendix
\appendixpage
\section{Proof of Theorem \ref{thm: equivalence}}
    \begin{theorem}\label{thm: equivalence1}
	For any two teaching instance defined by $(M, L, Q_0, \pi^\dagger)$ and $(M, L', Q'_0, \pi^\dagger)$, where $L$ and $L'$ are two Q-learning agent, if $Q_0$ and $Q'_0$ satisfies that $Q_0(s,a)\geq Q_0(s,a')$ if and only if $Q'_0(s,a)\geq Q'_0(s,a')$ for all $s\in\S$, $a,a'\in \A$, then $TD(M, L, Q_0, \pi^\dagger) = TD(M, L', Q_0', \pi^\dagger)$.
\end{theorem}
\begin{proof}
	We denote $Q_t \equiv Q'_t$ if $Q_0t(s,a)\geq Q_t(s,a')$ if and only if $Q'_t(s,a)\geq Q'_t(s,a')$ for all $s\in\S$, $a,a'\in \A$. The proof is based on a key technical lemma:
	\begin{lemma}\label{thm: lemma3}
		For any teaching policy $\phi$ for teaching instance $(M, L, Q_0, \pi^\dagger)$, there exists a matching teaching policy $\phi'$ for teaching instance $(M, L', Q'_0, \pi^\dagger)$, such that for any time $t$, if $Q_t \equiv Q'_t$ and $\xi_t = \xi'_t$, then $Q_{t+1} \equiv Q'_{t+1}$.
	\end{lemma}
Lemma \ref{thm: lemma3} implies that with the same random seed, the relationship $Q_{t+1} \equiv Q'_{t+1}$ will remain invariant through teaching when $\phi$ is used on $(M, L, Q_0, \pi^\dagger)$ and $\phi'$ used on $(M, L', Q'_0, \pi^\dagger)$. Thus, $Q_t\in\mathcal Q^\dagger$ if and only if $Q'_t\in\mathcal Q^\dagger$. Therefore, if there exists an optimal $\phi$ that achieves the teaching dimension for $(M, L, Q_0, \pi^\dagger)$, then the matching $\phi'$ also achieves the same teaching dimension for $(M, L, Q_0, \pi^\dagger)$. Thus, the $TD$ for both teaching instances much match. What remains is to prove Lemma \ref{thm: lemma3}.
\end{proof}
\begin{proof}[Proof of Lemma \ref{thm: lemma3}]
	Given a particular $Q_t$ and an experience tuple $\xi_t = (s_t, a_t, s_t')$, the Q-learning update rule will only modify the value of $Q_t(s_t,a_t)$ based on the teacher provided reward $r_t = \phi(\xi_t)$. Define the rank of $(s_t,a_t)$ in $Q_t$ as $rank_{Q_t}(s_t,a_t) = |\{a| Q_t(s_t,a)>Q_t(s_t,a_t)\}|$. Under the Q-learning update rule 
	\begin{equation}
	Q_{t+1}(s_t,a_t) = (1-\alpha)Q_t(s_t,a_t) + \alpha (r_t + \gamma \max_{a'} Q_t(s_{t+1},a')),
	\end{equation}
	it is obvious that there exist $r_t$ such that the rank of $(s_t,a_t)$ in $Q_{t+1}$ can be any of $0,1,...,A-1$. Assume now that the rank of $(s_t,a_t)$ in $Q_{t+1}$ becomes $k$ after updating with $r_t = \phi(\xi_t)$. Define $\phi'$ such that $r'_t = \phi'(\xi'_t)$ will also update the rank of $(s_t,a_t)$ in $Q'_{t+1}$ to be $k$. Then, since $Q_t \equiv Q'_t$ and for both $Q_{t+1}$ and $Q'_{t+1}$ the rank of $(s_t,a_t)$ becomes $k$, we have $Q_{t+1} \equiv Q'_{t+1}$. This concludes the proof.
\end{proof}

\section{Experiment Setting and Hyperparameters for TD3}
\label{sec:TD3params}
Throughout the experiments, we use the following set of hyperparameters for TD3, described in Table \ref{table:hyperparameters}. The hyperparameters are selected via grid search on the Dog MDP. Each experiment is run for 5000 episodes, where each episode is of 200 iteration long. The learned policy is evaluated for every $50$ episodes, and the policy with the best evaluation performance is used for the computation of the teaching dimension. In the computation of the teaching dimension, we run the best-found TD3 policy for 1000 episodes, and take the average number of steps to teach the target policy. This gives $11.0$.
\begin{table*}[t]
	\centering
	\begin{tabular}{| l | c | l |} 
		\hline
		Parameters & Values & Description\\ [0.5ex] 
		\hline\hline
		exploration noise & $0.5$ & Std of Gaussian exploration noise.\\
		batch size & 100 & Batch size for both actor and critic\\
		discount factor & 0.99 & Discounting factor for the attacker problem.\\
		policy noise & 0.2 & Noise added to target policy during critic update.\\
		noise clip & $[-0.5, 0.5]$ & Range to clip target policy noise.\\
		action L2 weight & 50 & Weight for L2 regularization added to the actor network loss function.\\
		buffer size & $10^7$ & Replay buffer size, larger than total number of iterations.\\
		optimizer & Adam & Use the Adam optimizer.\\
		learning rate critic & $10^{-3}$ & Learning rate for the critic network.\\
		learning rate actor & $5^{-4}$ & Learning rate for the actor network.\\
		$\tau$ & $0.002$ & Target network update rate.\\
		policy frequency & 2& Frequency of delayed policy update.\\
		\hline
	\end{tabular}
	\caption{Hyperparameters for TD3.}
	\label{table:hyperparameters}
\end{table*}
\section{Datasets}
   
\subsection{Data Collection Process}

Participants accessed the study on Amazon MTurk. Each participant was first given an instruction about the task (Fig.~\ref{fig:instruction}) and was required to complete an extensive quiz about the instructions (Fig.~\ref{fig:quiz}) and could not continue until every question was answered correctly.
Participants trained three dogs and then was given a short survey to ensure that they treated the slider symmetrically and gather standard demographic data. Each participant was reimbursed for \$2.00 for taking the task.
         
\subsection{Participants and Data Exclusion}
We recruited 791 participants through Amazon MTurk. 
The number of participants was chosen {\em a priori} based on the authors' intuition from similar previously conducted studies.  
We excluded 42 participants (11 for failing to complete the task, 19 due to experiment error, and 12 who selected ``do nothing'' for at least 37 steps out of 40 steps for training a single dog). 
In addition, the dogs that were trained successfully with steps less than the optimal length (the number of steps the optimal teaching policy needs) were excluded from the analysis. For example, if the subject gave punishment on every step and the dog luckily went to the left at every tile all the way from the rightmost tile to the leftmost tile, the dog could be trained in just 4 steps. These data were excluded because they did not really reflect a success in training, but more likely a misunderstanding of the goal and luck. This composes 9.03 \% of the total dogs (203 out of 2247 dogs).  In sum, there are 749 subjects (2044 dogs) in the analysis.

% The figure of the instruction
\begin{figure}[!t]
    \centering
    \begin{subfigure}{1\textwidth}
        \includegraphics[width=1.1\linewidth]{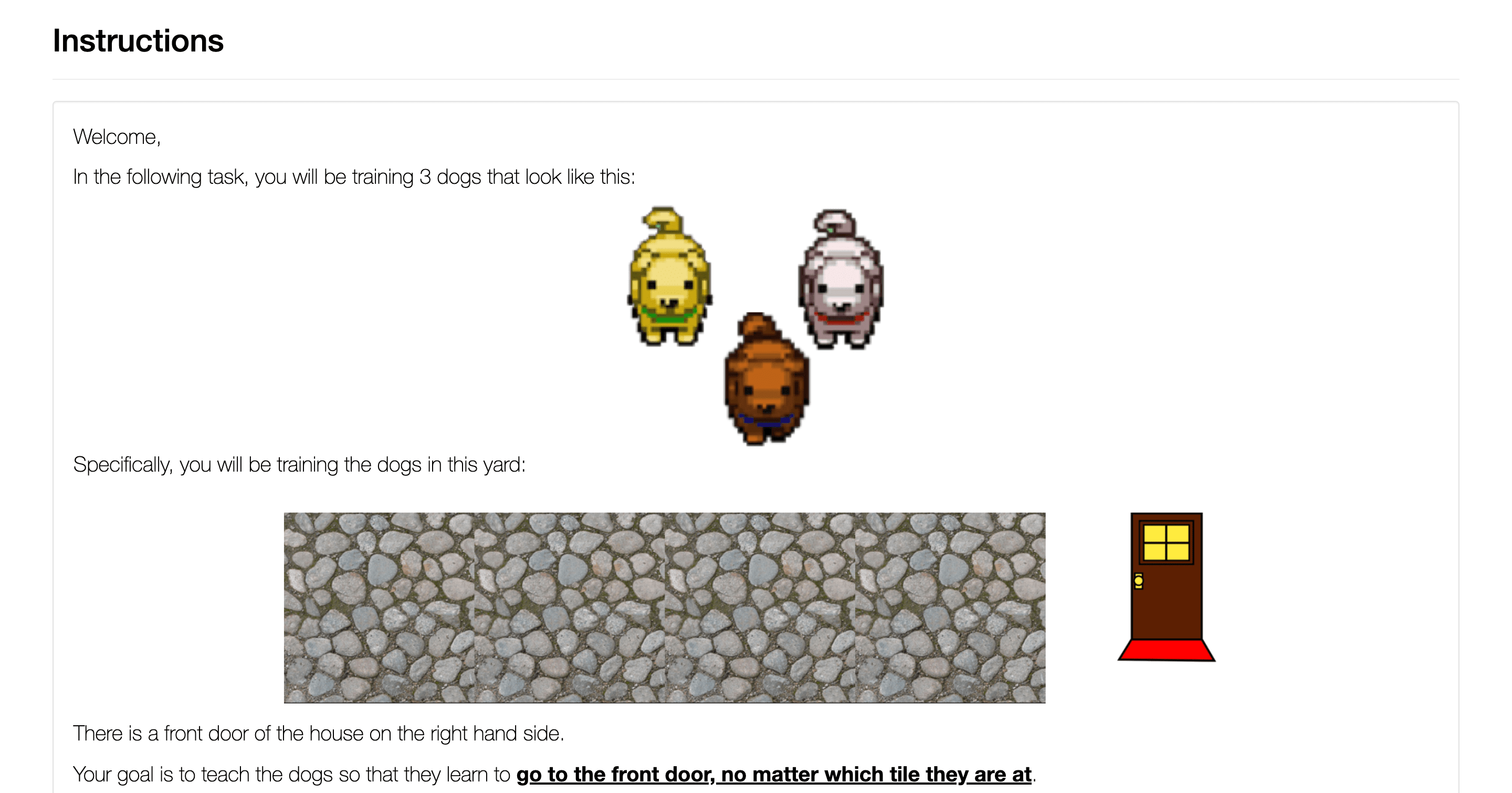}
    \end{subfigure}

\medskip
    \begin{subfigure}{1\textwidth}
        \includegraphics[width=1.1\linewidth]{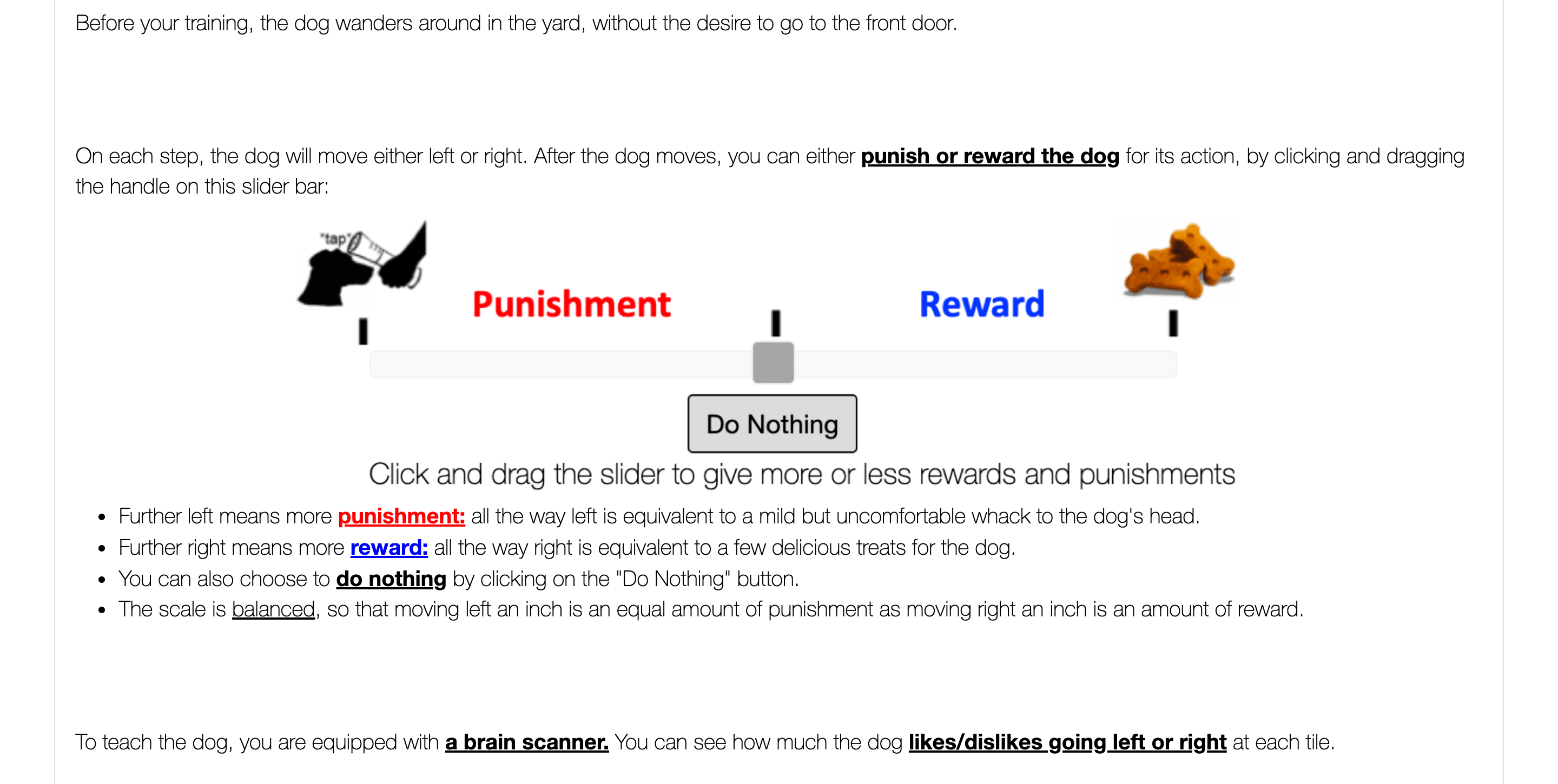}
    \end{subfigure}
    \caption[]{The instruction before taking the dog training task.}
\end{figure}

\begin{figure}[ht]\ContinuedFloat
    \centering
    \begin{subfigure}{1\textwidth}
        \includegraphics[width=1.1\linewidth]{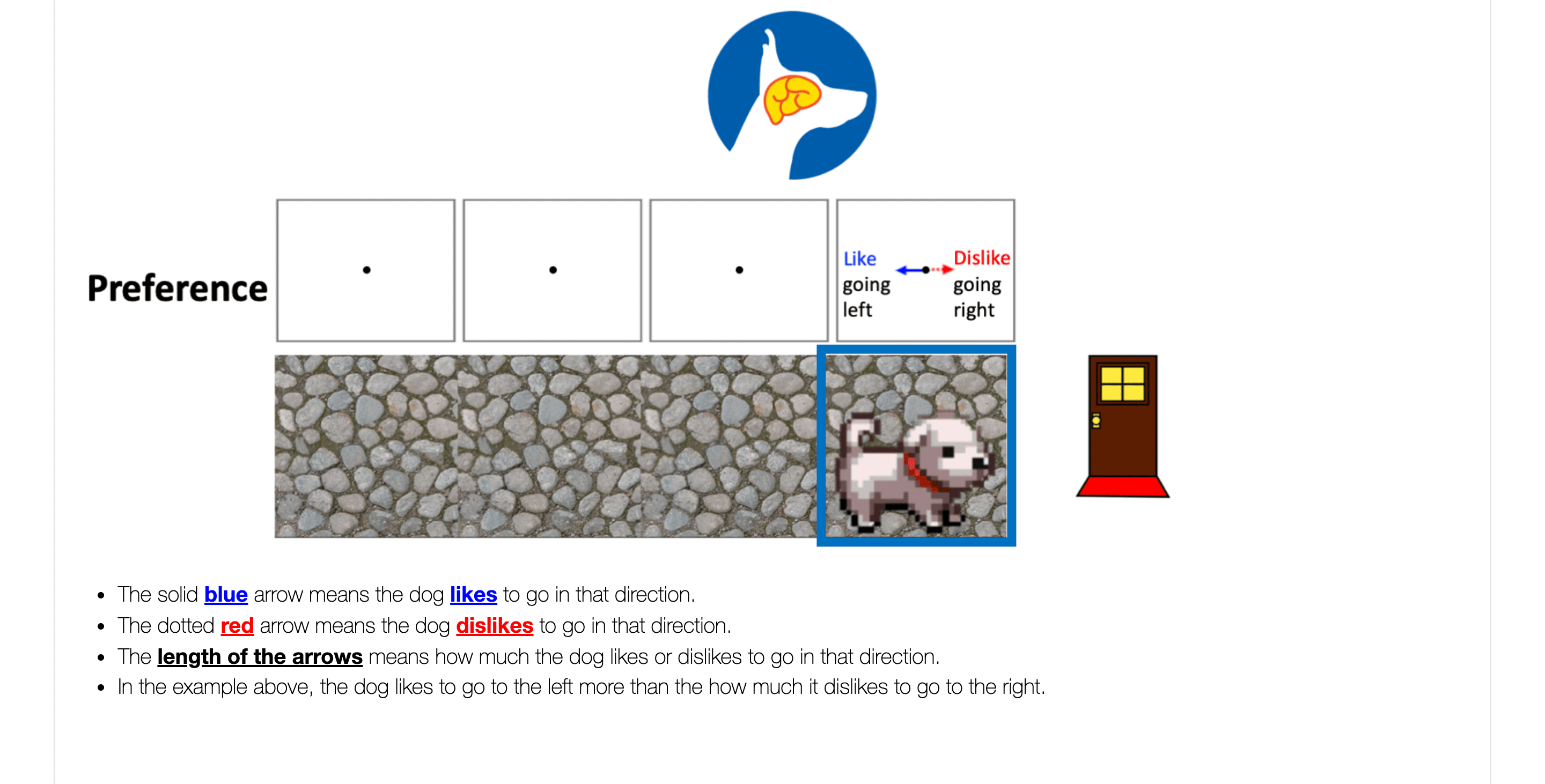}
    \end{subfigure}

\medskip
    \begin{subfigure}{1\textwidth}
        \includegraphics[width=1.1\linewidth]{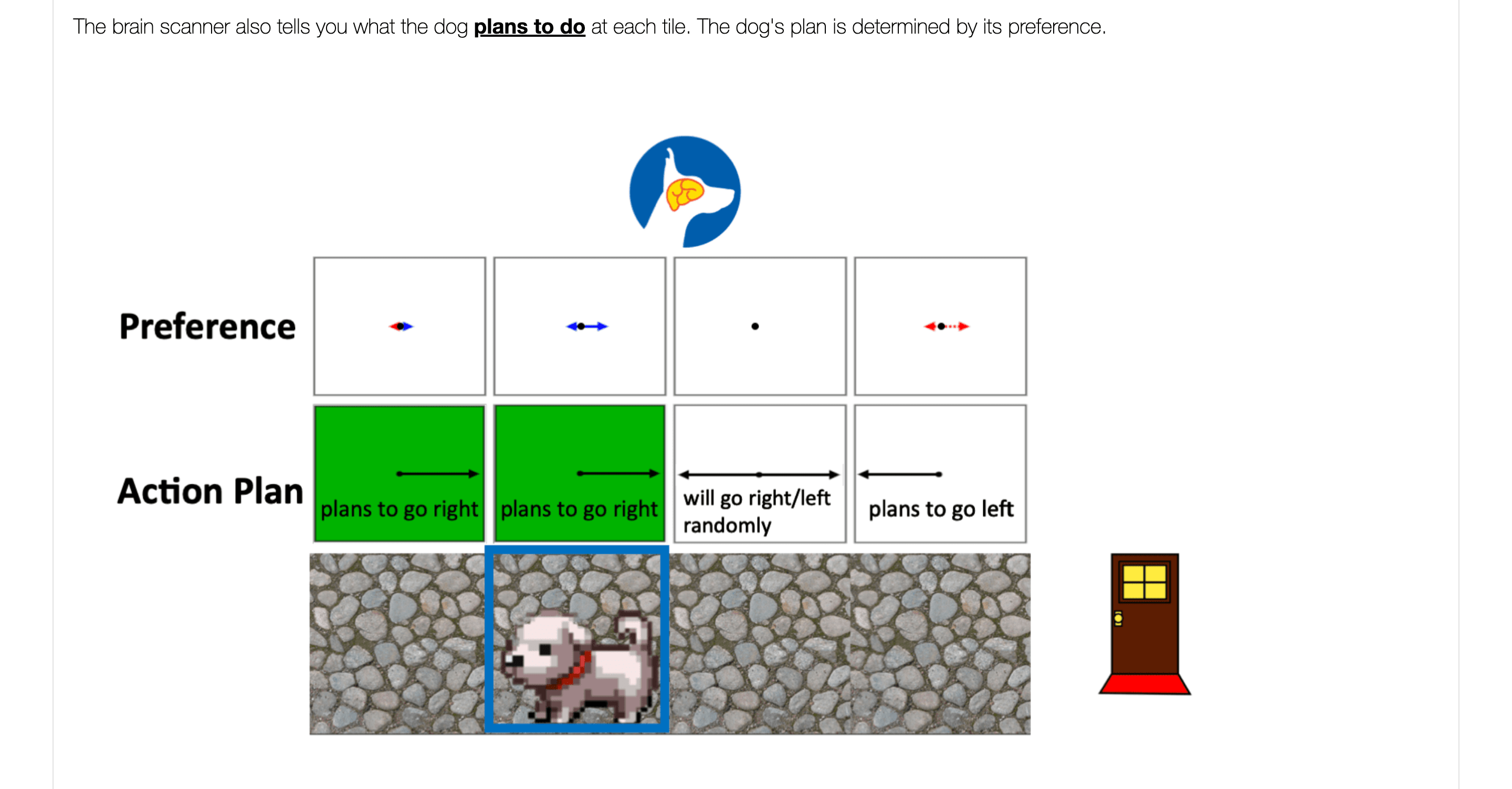}
    \end{subfigure}

    \begin{subfigure}{1\textwidth}
        \includegraphics[width=1.1\linewidth]{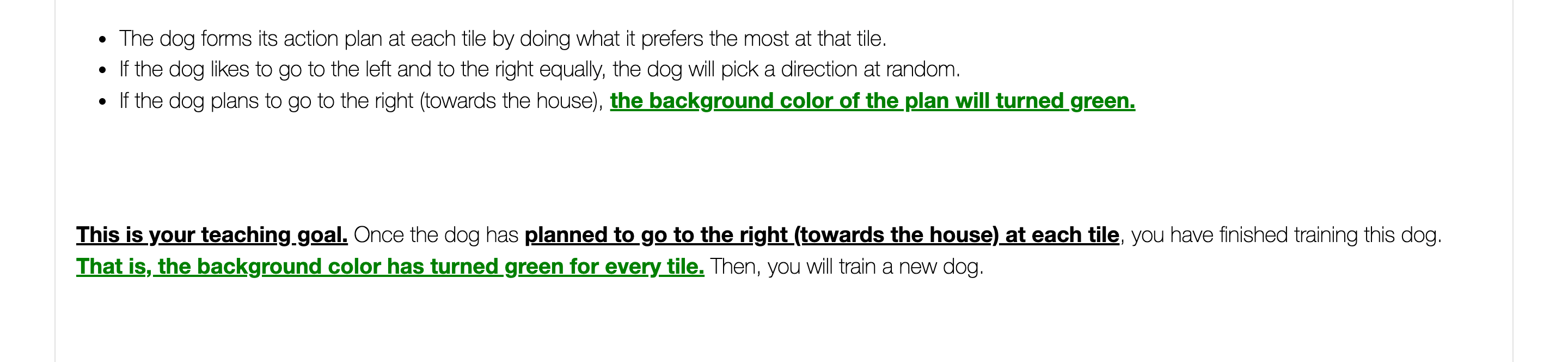}        
    \end{subfigure}    
    \caption[]{The instruction before taking the dog training task (continued).}    
\end{figure}

\begin{figure}[ht]\ContinuedFloat
    \centering    
    \begin{subfigure}{1\textwidth}
        \includegraphics[width=1.1\linewidth]{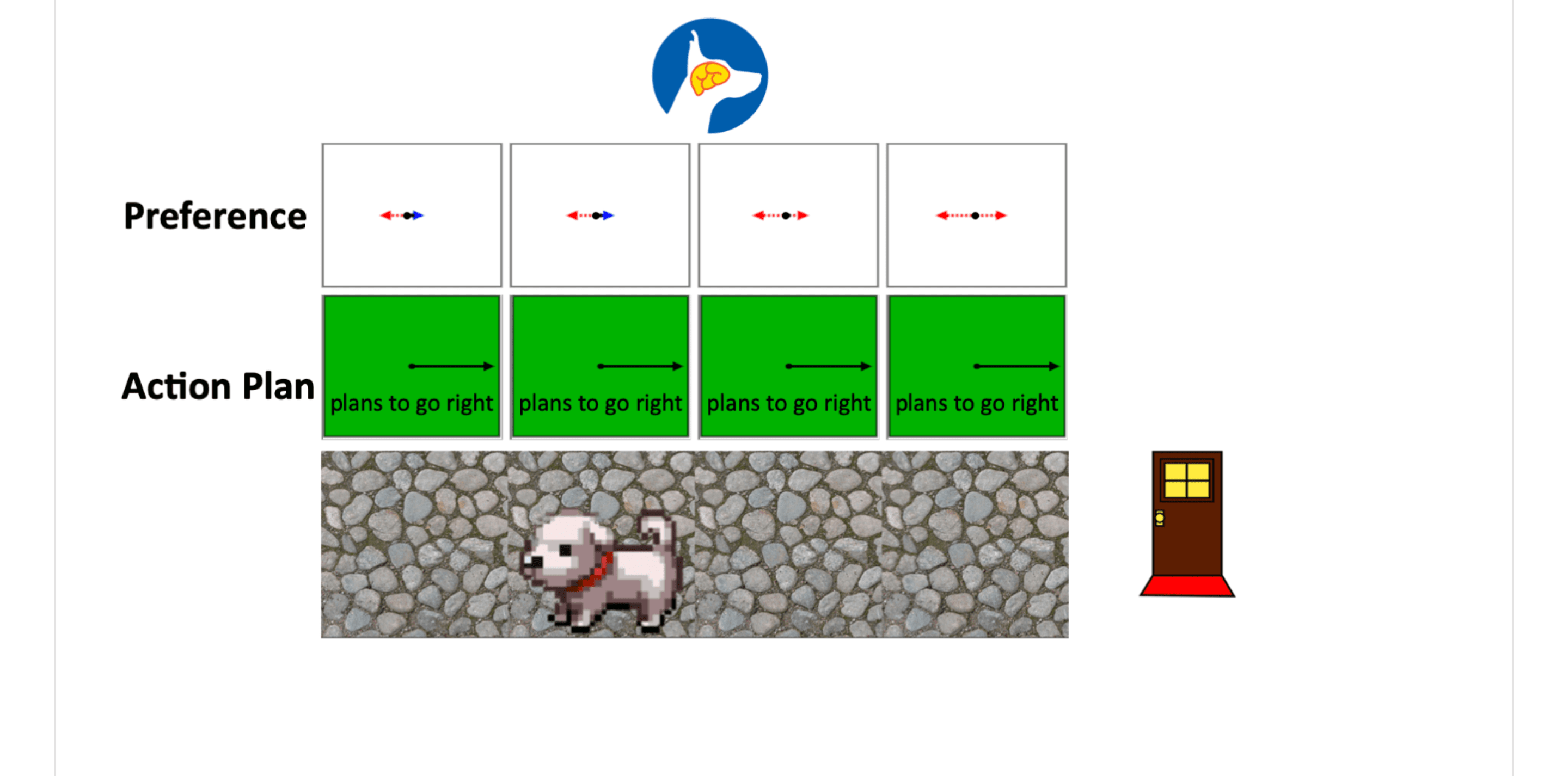}        
    \end{subfigure}        
    \begin{subfigure}{1\textwidth}
        \includegraphics[width=1.1\linewidth]{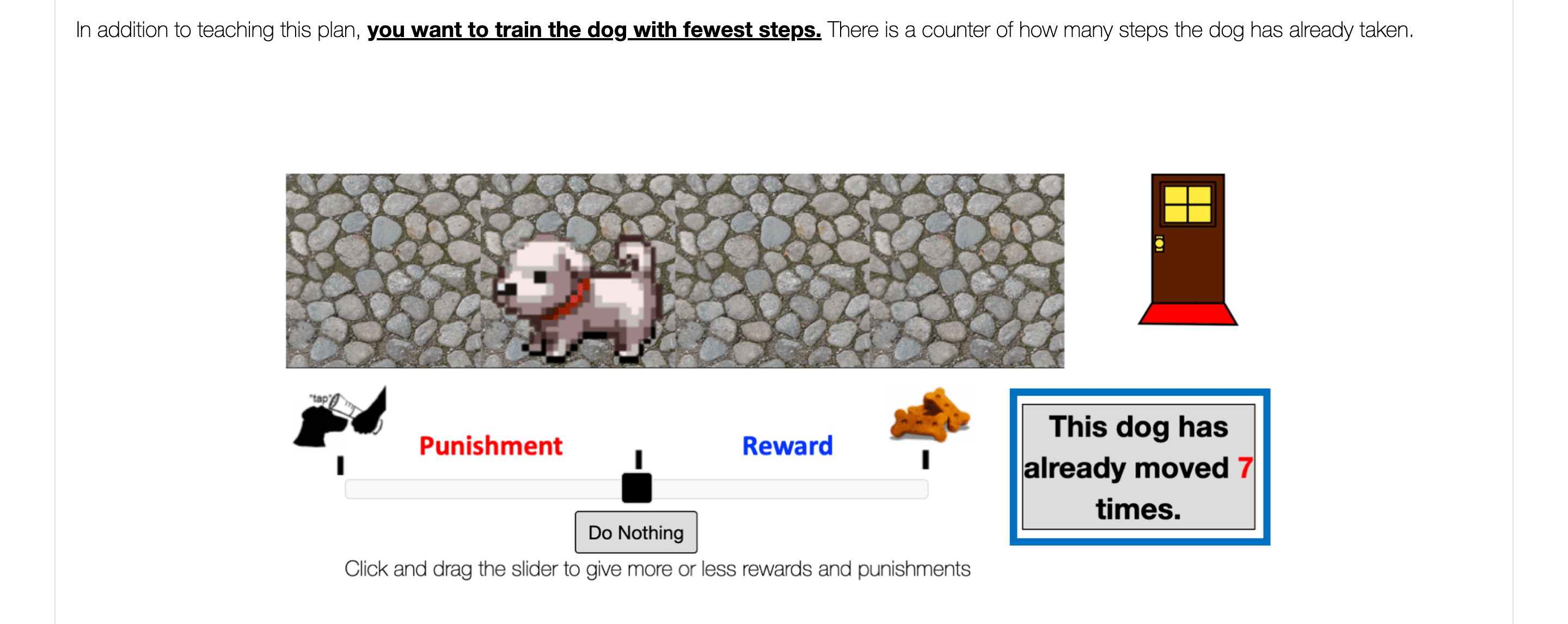}        
    \end{subfigure}          
    \begin{subfigure}{1\textwidth}
        \includegraphics[width=1.1\linewidth]{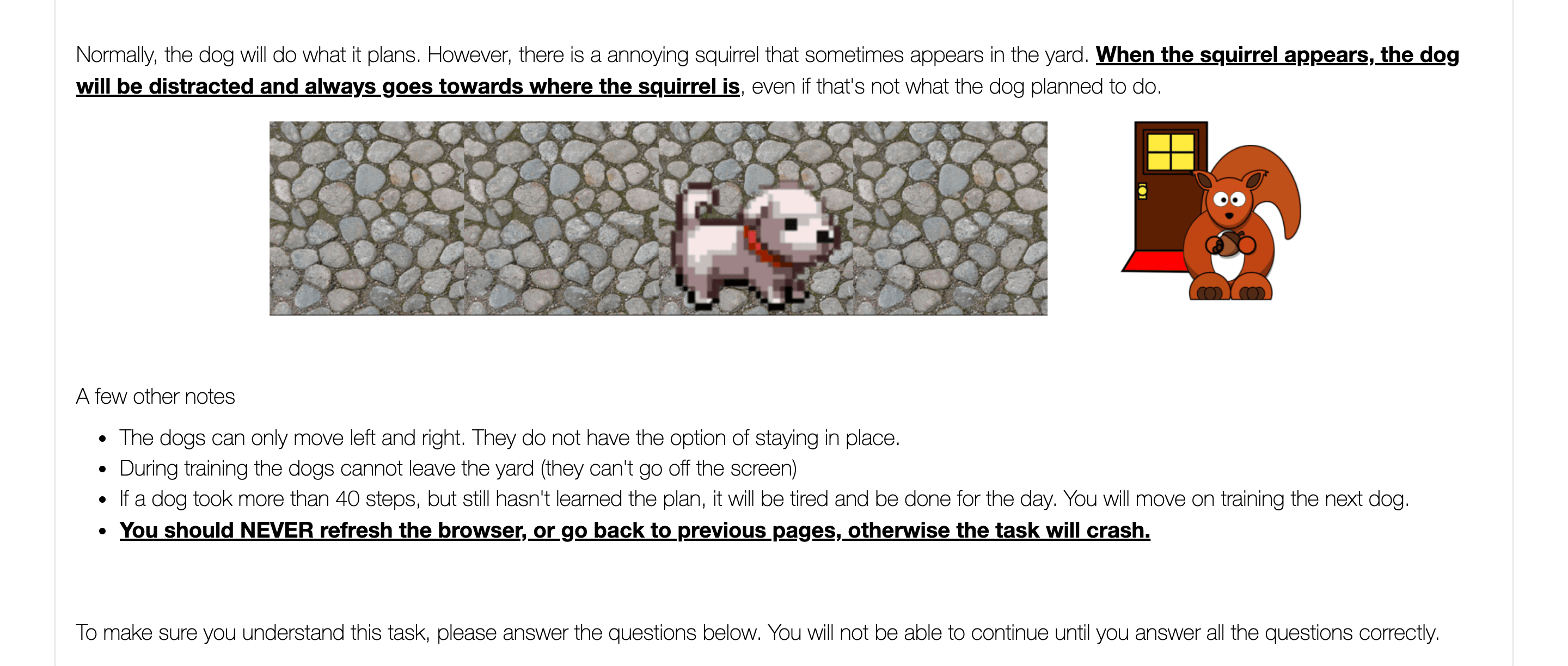}        
    \end{subfigure}              
    \caption[]{The instruction before taking the dog training task (continued).}
    \label{fig:instruction}
\end{figure}

% The figure of the quiz
\begin{figure}[!t]
    \centering
    \begin{subfigure}{1\textwidth}
        \includegraphics[width=1.1\linewidth]{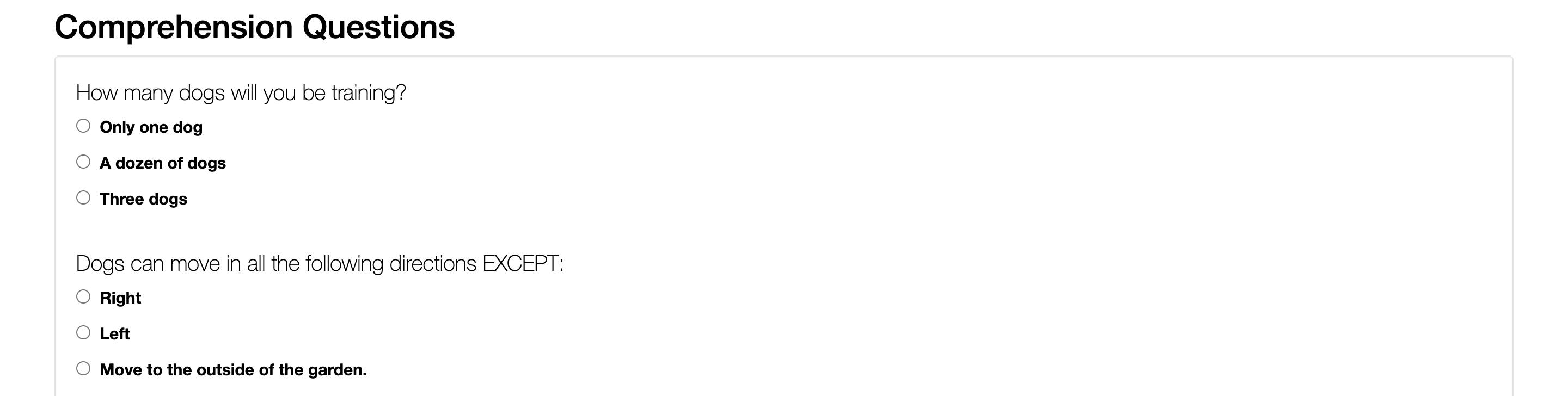}
    \end{subfigure}

\medskip
    \begin{subfigure}{1\textwidth}
        \includegraphics[width=1.1\linewidth]{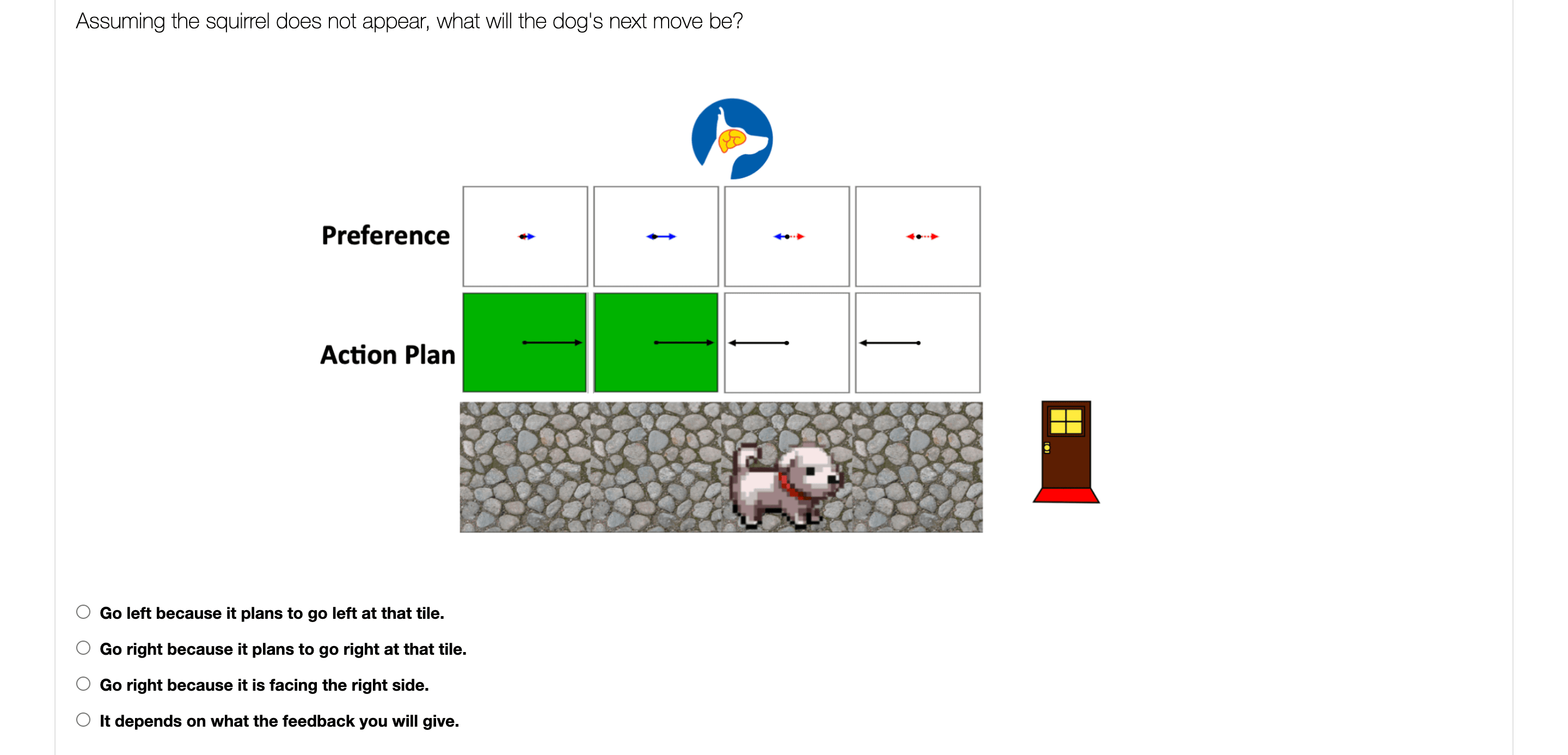}
    \end{subfigure}
    \caption[]{The quiz before taking the dog training task.}
\end{figure}

\begin{figure}[ht]\ContinuedFloat
    \centering
    \begin{subfigure}{1\textwidth}
        \includegraphics[width=1.1\linewidth]{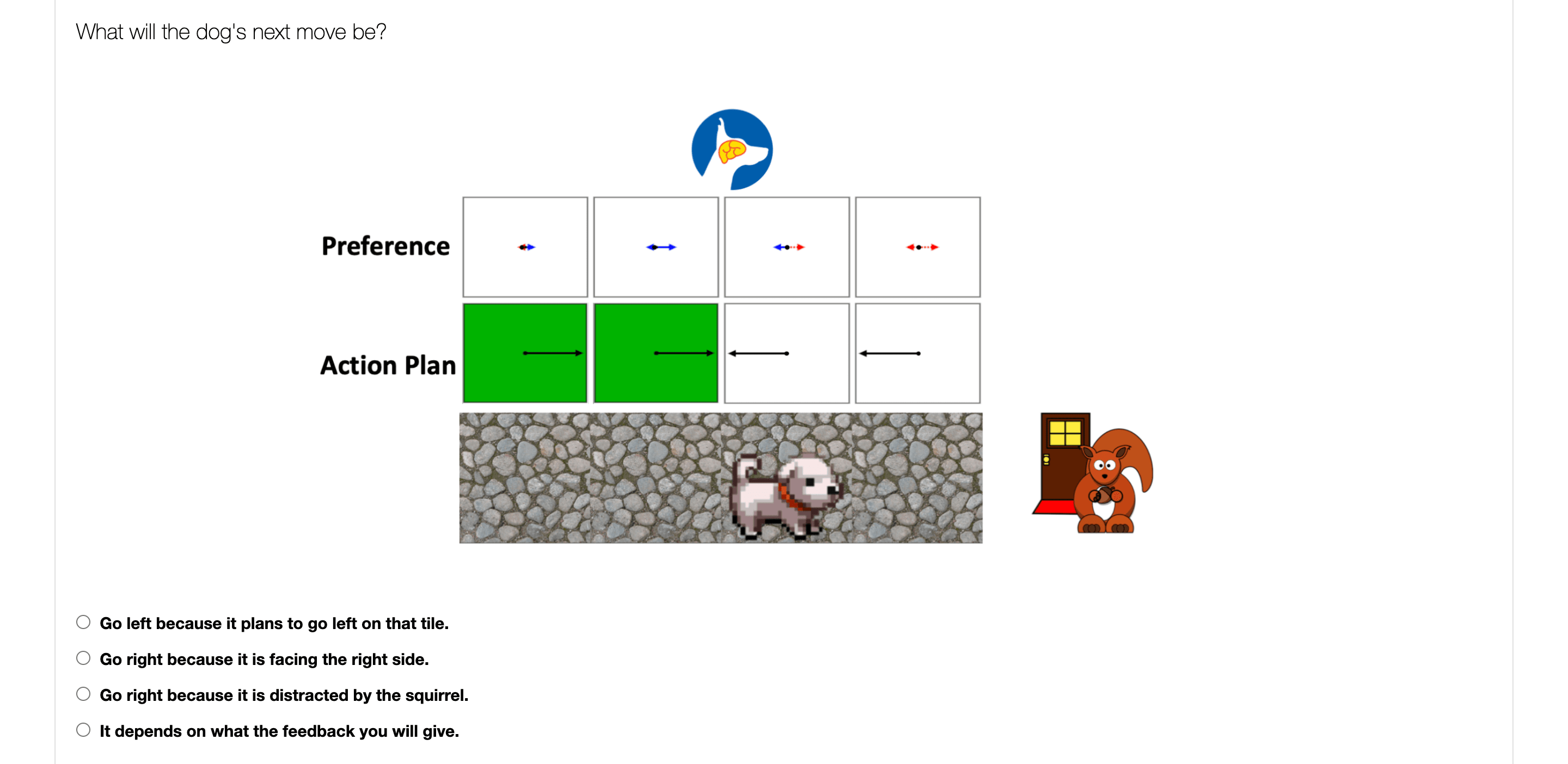}
    \end{subfigure}

\medskip
    \begin{subfigure}{1\textwidth}
        \includegraphics[width=1.1\linewidth]{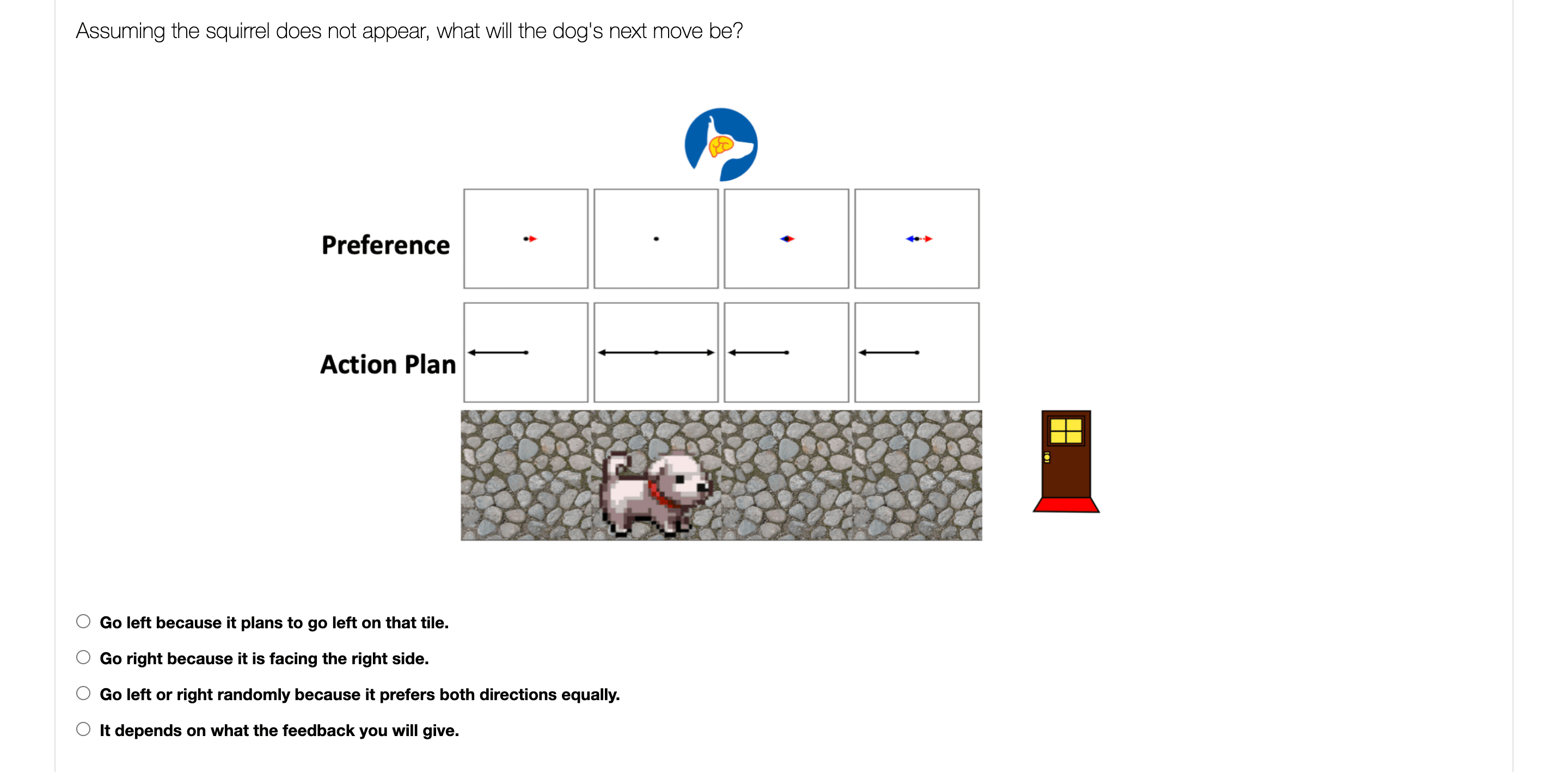}
    \end{subfigure}

    \begin{subfigure}{1\textwidth}
        \includegraphics[width=1.1\linewidth]{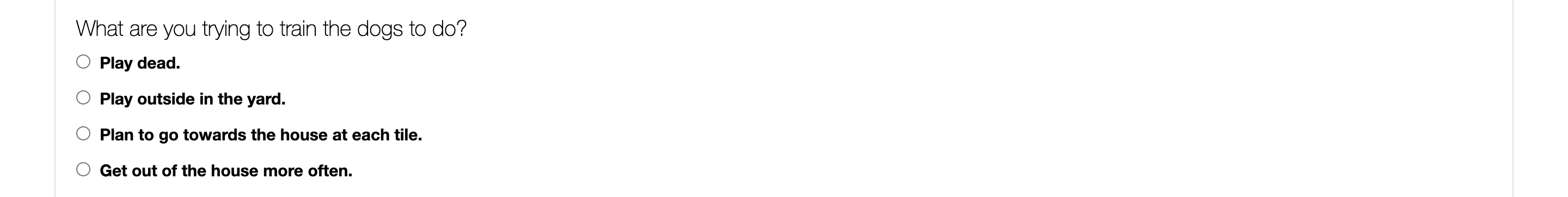}        
    \end{subfigure}    
    \caption[]{The quiz before taking the dog training task (continued).}    
\end{figure}

\begin{figure}[ht]\ContinuedFloat
    \centering
    \begin{subfigure}{1\textwidth}
        \includegraphics[width=1.1\linewidth]{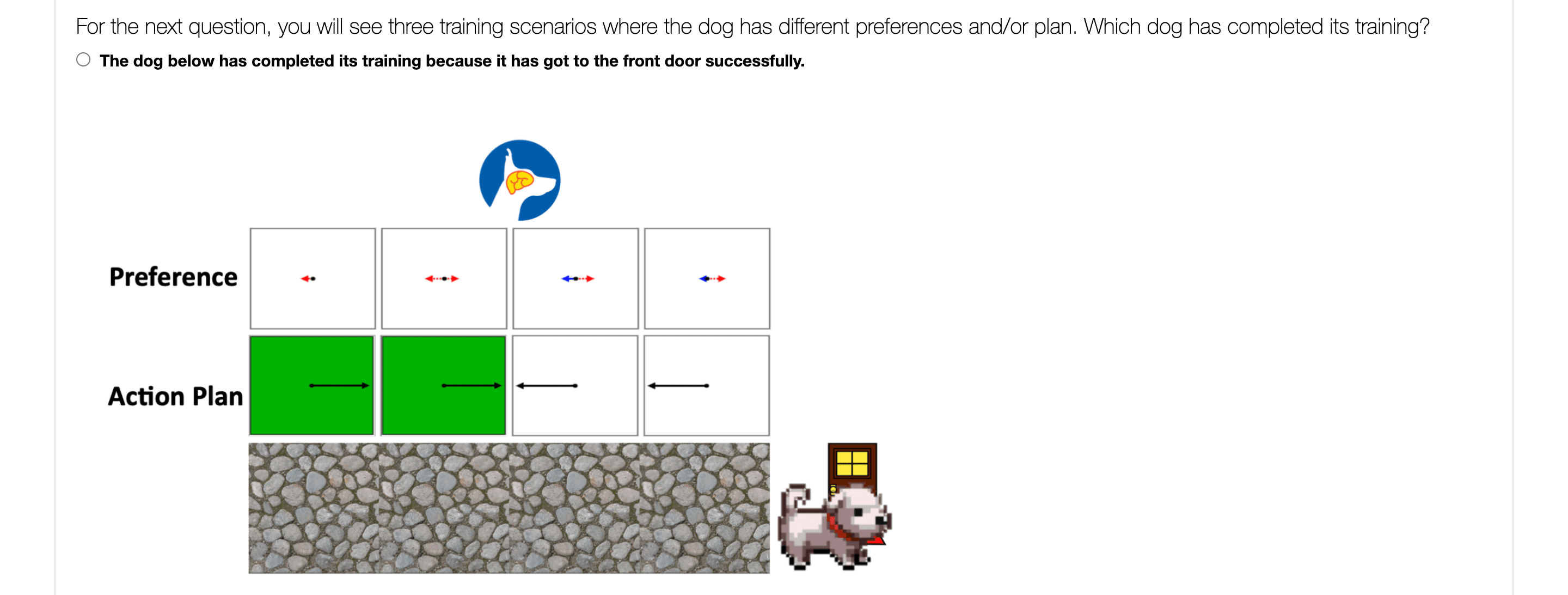}
    \end{subfigure}

\medskip
    \begin{subfigure}{1\textwidth}
        \includegraphics[width=1.1\linewidth]{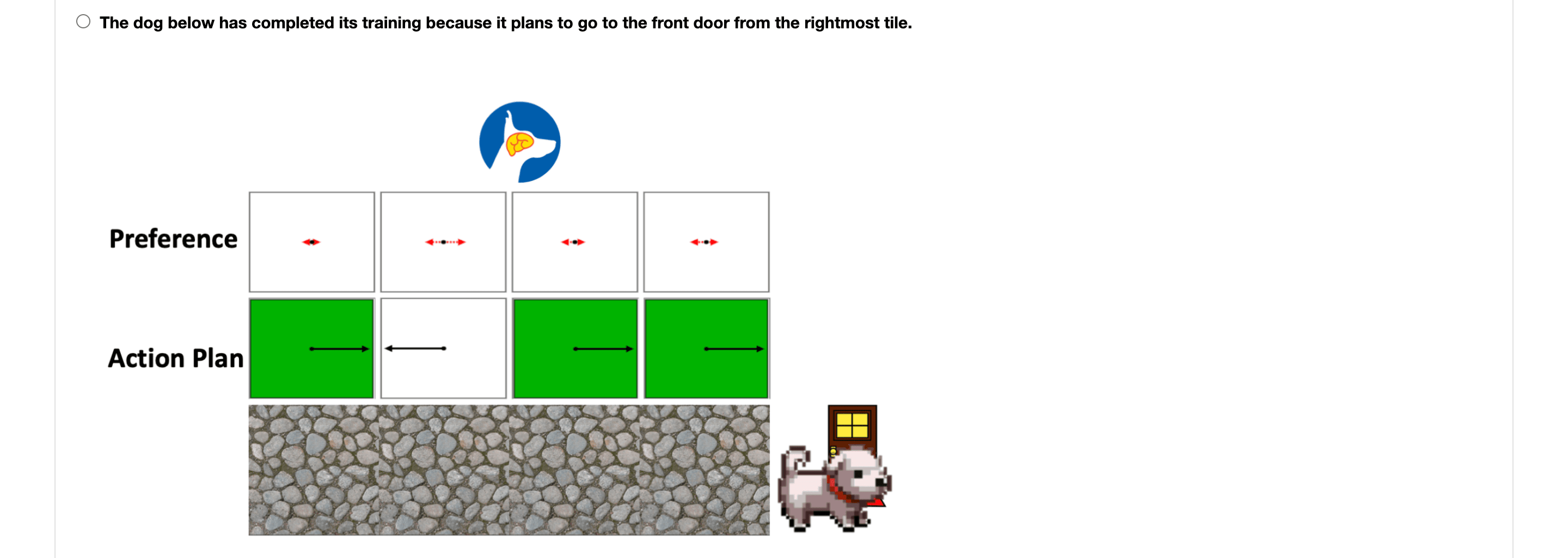}
    \end{subfigure}

    \begin{subfigure}{1\textwidth}
        \includegraphics[width=1.1\linewidth]{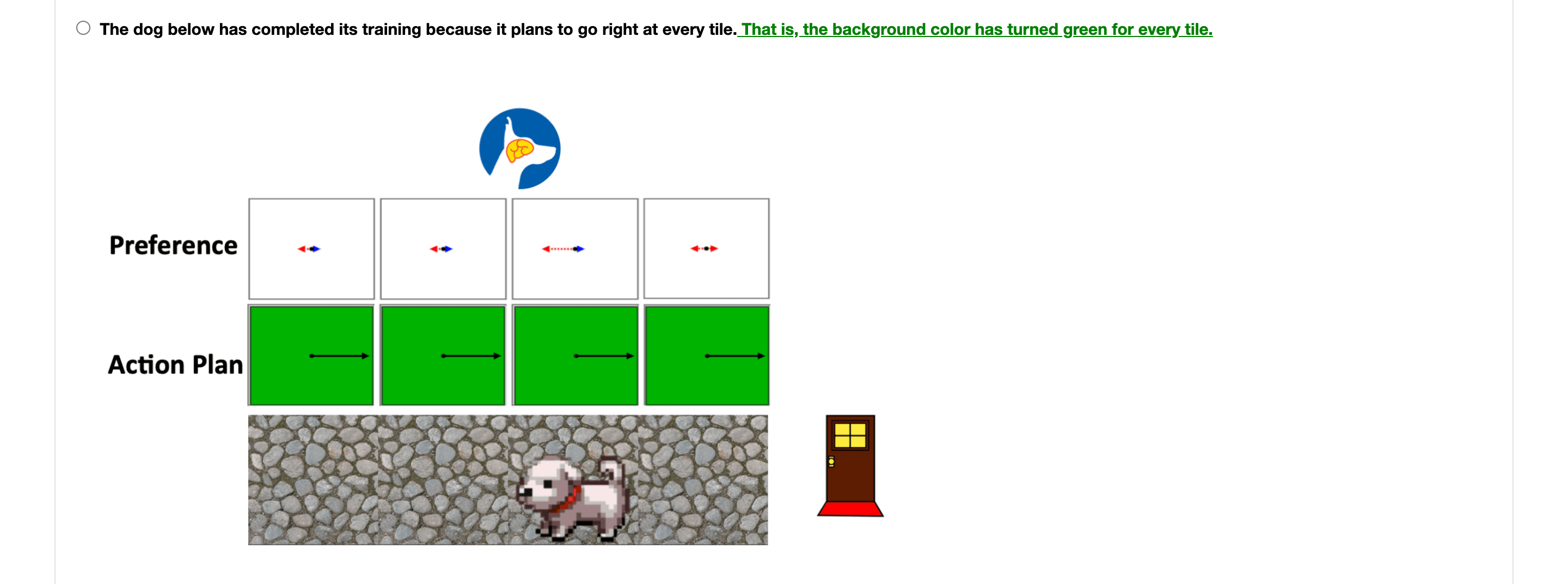}        
    \end{subfigure}    
    \caption[]{The quiz before taking the dog training task (continued).}    
\end{figure}

\begin{figure}[ht]\ContinuedFloat
    \centering    
    \begin{subfigure}{1\textwidth}
        \includegraphics[width=1.1\linewidth]{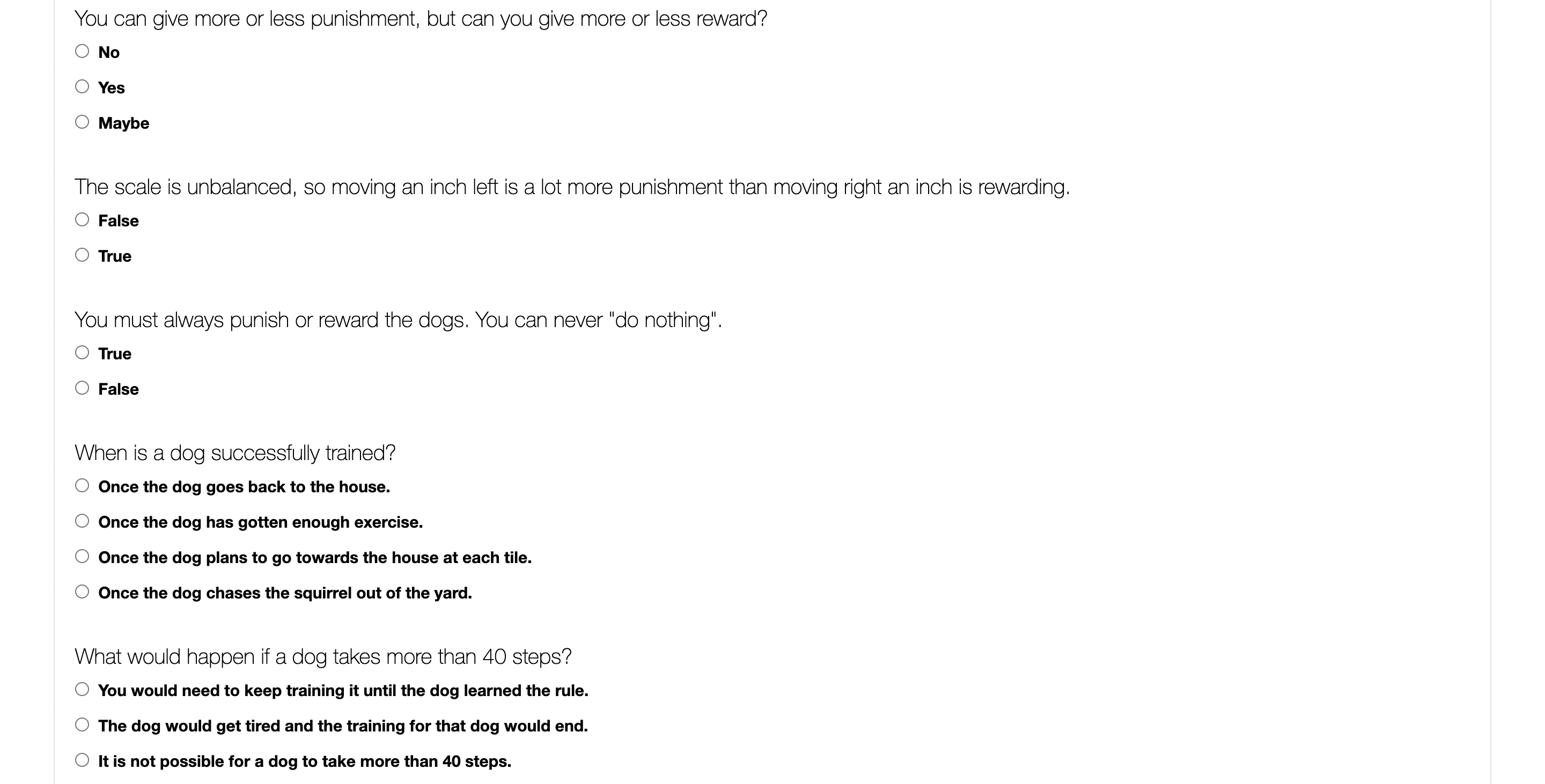}        
    \end{subfigure}        
    \begin{subfigure}{1\textwidth}
        \includegraphics[width=1.1\linewidth]{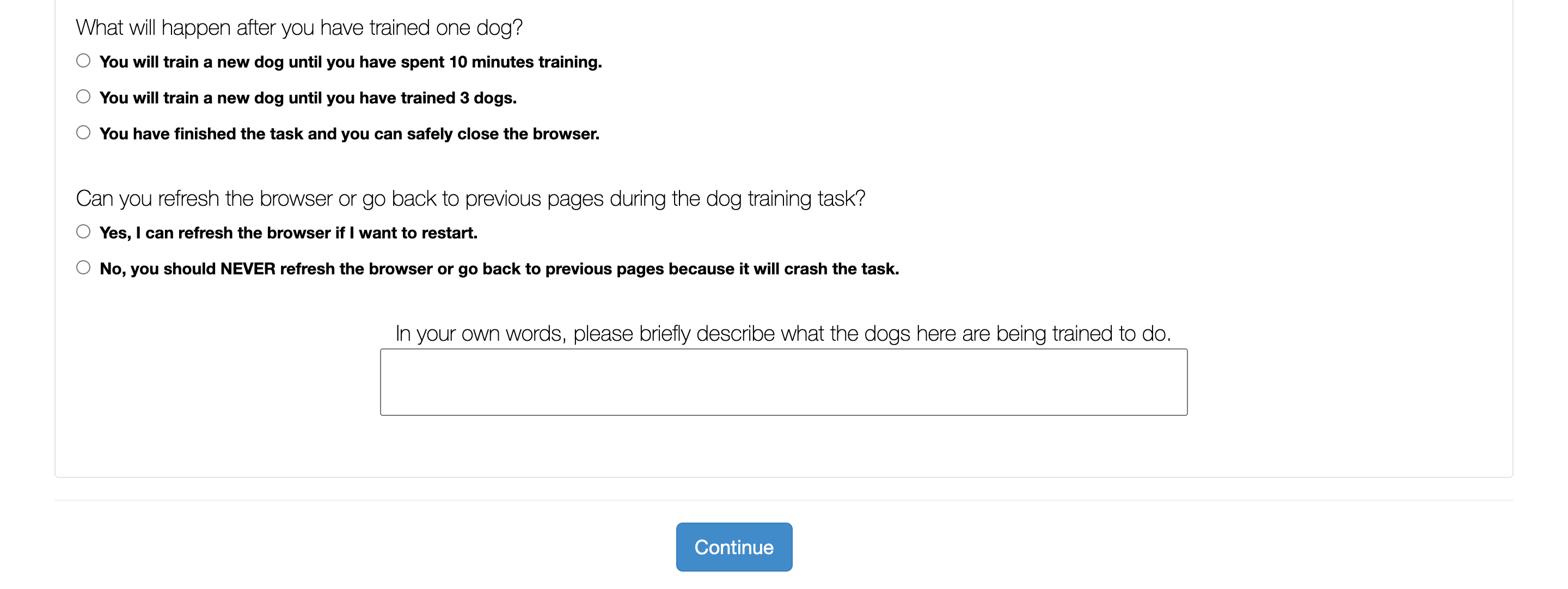}        
    \end{subfigure}          
    \caption[]{The quiz before taking the dog training task (continued).}
    \label{fig:quiz}
\end{figure}

\end{document}